%% file: correlated_latents_icml.tex
\documentclass{article}

\usepackage{microtype}
\usepackage{graphicx}
\usepackage{subfigure}
\usepackage{booktabs} 
\usepackage{multirow}
\usepackage[table,xcdraw]{xcolor}
\usepackage[tableposition=top]{caption}
\usepackage{hyperref}

\usepackage[accepted]{icml2024}

\usepackage{amsmath}
\usepackage{amssymb}
\usepackage{mathtools}
\usepackage{amsthm}
\usepackage{mathrsfs}
\usepackage[capitalize,noabbrev]{cleveref}

\usepackage{bbm}
\usepackage{bm}  
\usepackage[british]{babel}
\usepackage[utf8]{inputenc}
\usepackage{comment}
\usepackage[inline,shortlabels]{enumitem}
\usepackage{graphbox}       
\PassOptionsToPackage{hyphens}{url}  
\usepackage{mathtools}
\usepackage{microtype}      
\usepackage{nicefrac}
\usepackage{physics}
\usepackage{wrapfig}
\usepackage{listings}
\usepackage[textsize=tiny]{todonotes}

\icmltitlerunning{Correlated latent variables accelerate learning with neural
  networks}

\input{settings}

\begin{document}

\twocolumn[
\icmltitle{Sliding Down the Stairs: How Correlated Latent Variables\\ Accelerate
  Learning with Neural Networks}

\icmlsetsymbol{equal}{*}

\begin{icmlauthorlist}
\icmlauthor{Lorenzo Bardone}{sissa}
\icmlauthor{Sebastian Goldt}{sissa}
\end{icmlauthorlist}

\icmlaffiliation{sissa}{International School of Advanced Studies, Trieste, Italy}

\icmlcorrespondingauthor{Lorenzo Bardone}{lbardone@sissa.it}
\icmlcorrespondingauthor{Sebastian Goldt}{sgoldt@sissa.it}

\icmlkeywords{Machine Learning, ICML}

\vskip 0.3in
]



\printAffiliationsAndNotice{}  

\begin{abstract}
  \input{content/abstract}
\end{abstract}
\input{content/main}

\subsection*{Acknowledgements}

\input{content/acknowledgements}

\subsection*{Code availability} We provide code to reproduce our experiments  via GitHub at 
\href{https://github.com/anon/correlatedlatents}{https://github.com/anon/correlatedlatents}.

\subsection*{Impact statement}

This paper presents work whose goal is to advance the field of Machine Learning. There are many potential societal consequences of our work, none which we feel must be specifically highlighted here.

\bibliography{correlated_latents_icml}
\bibliographystyle{icml2024}

\newpage
\appendix
\onecolumn

\input{content/appendix}

\end{document}

%% file: settings.tex
\renewcommand{\P}{{\mathbb P}}
\def\Q{{\mathbb Q}}

\newcommand{\R}{{\mathbb R}}
\newcommand{\N}{{\mathbb N}}
\newcommand{\E}{{\mathbb E}}

\definecolor{C0}{HTML}{1f77b4}
\definecolor{C1}{HTML}{ff7f0e}
\definecolor{C2}{HTML}{2ca02c}
\definecolor{C3}{HTML}{d62728}
\definecolor{C4}{HTML}{9467bd}
\definecolor{C5}{HTML}{8c564b}
\definecolor{C6}{HTML}{e377c2}
\definecolor{C7}{HTML}{7f7f7f}
\definecolor{C8}{HTML}{bcbd22}
\definecolor{C9}{HTML}{17becf}
\definecolor{gaussian}{HTML}{47908C}

\newtheorem{theorem}{Theorem}
\newtheorem{lemma}[theorem]{Lemma}
\newtheorem{proposition}[theorem]{Proposition}
\newtheorem{remark}[theorem]{Remark}

\newtheorem{assumption}{Assumption}

\def\Q{\mathbb Q}
\def\P{\mathbb P}
\def\E{\mathbb E}
\def\N{\mathbb N}
\def\R{\mathbb R}
\def\amult{{\boldsymbol{\alpha}}}

\def\nablasph{\nabla_{\mathrm{sph}}}

\def\id{\mathbbm{1}}

\newcommand{\EE}{\mathbb{E}\,}

\newcommand{\reals}{\mathbb{R}}

\def\id{\mathbbm{1}}

%% file: content/abstract.tex
Neural networks extract features from data using stochastic gradient descent
(SGD). In particular, higher-order input cumulants (HOCs) are crucial for their performance. However, extracting information from the $p$th cumulant of~$d$\nobreakdash-dimensional inputs is computationally hard: the number of samples required to recover
a single direction from an order\nobreakdash-$p$ tensor  (tensor PCA) using online SGD grows as $d^{p-1}$, which is prohibitive for high-dimensional inputs. This result raises the question of how neural networks extract relevant directions from the HOCs of their inputs \emph{efficiently}. Here, we show that correlations between latent variables along the directions encoded in different input cumulants
speed up learning from higher-order correlations. We show this effect
analytically by deriving nearly sharp thresholds for the number of samples required by a single neuron to weakly-recover these directions using online SGD from a random start in high dimensions. Our analytical results are confirmed in simulations of two-layer neural networks and unveil a new mechanism for hierarchical learning in
neural networks.

%% file: content/main.tex
\section{Introduction}
Neural networks excel at learning rich representations of their
data, 
but which parts of a data set are actually important for them? 
From a statistical point of view, we can decompose the
data distribution into cumulants, which capture correlations between groups of
variables
.
The first cumulant is the mean, the
second describes pair-wise correlations, and \emph{higher-order
  cumulants} (HOCs) encode correlations between three or more variables. In
image classification, HOCs are particularly important: on
CIFAR10, removing HOCs from the training distribution incurs a drop in test
accuracy of up to 65 $\%$ for 
DenseNets, ResNets, and Vision
transformers~\cite{refinetti2023neural}.

While HOCs are important for the performance of neural networks, extracting information from HOCs with stochastic gradient descent (SGD) is computationally hard. Take the simple case of tensor PCA~\cite{richard2014statistical}, where one aims to recover a ``spike'' $u \in \reals^d$ from an order-$p$ tensor $T$, which could be the order-$p$ cumulant of the inputs. Modelling the tensor as $T = \beta u^{\otimes p} + Z$, with signal-to-noise ratio (SNR) $\beta > 0$ and noise tensor $Z$ with i.i.d.\ entries of mean zero and variance~1, \citet{benarous2021online} showed that online SGD with a single neuron requires a number of samples $n \gtrsim d^{p-1}$ to recover $u$. 
While smoothing the loss landscape can reduce to ~$n \gtrsim d^{p/2}$, as suggested by Correlational Statistical Query bounds (CSQ)~\citep{damian2022neural, abbe2023sgd}, a sample complexity that is exponential in $p$ is too expensive for high-dimensional inputs like images. 
For supervised learning, \citet{szekely2023learning} showed that the number of samples required to strongly distinguish
two classes of inputs $x\in\mathbb{R}^d$ whose distributions have the same mean and covariance, but different fourth-and higher-order cumulants, scales as $n \gtrsim d^2$ for the wide class of polynomial-time algorithms covered by the low-degree conjecture~\citep{barak2019nearly,
  hopkins2017power, hopkins2017bayesian, hopkins2018statistical}. Their numerical
experiments confirmed that two-layer neural networks indeed require quadratic
sample complexity to learn this binary classification task.

The theoretical difficulty of learning from HOCs is in apparent contradiction to the importance of HOCs for the performance of neural networks, and the speed with which neural networks pick them up~\cite{ingrosso2022data, refinetti2023neural, merger2023learning, belrose2024neural}, 
lead us to the following question: 
\begin{quote}
  \itshape \centering 
  How do neural networks extract information from
  higher-order input correlations \emph{efficiently}?
\end{quote}
In this paper, we show that neural networks can learn efficiently from higher-order cumulants by exploiting \textbf{correlations between the latent variables of type $u \cdot x$ corresponding to input cumulants of different orders}. 

\begin{figure*}[t!]
  \centering
  \includegraphics[trim=0 620 75 0,clip,width=\linewidth]{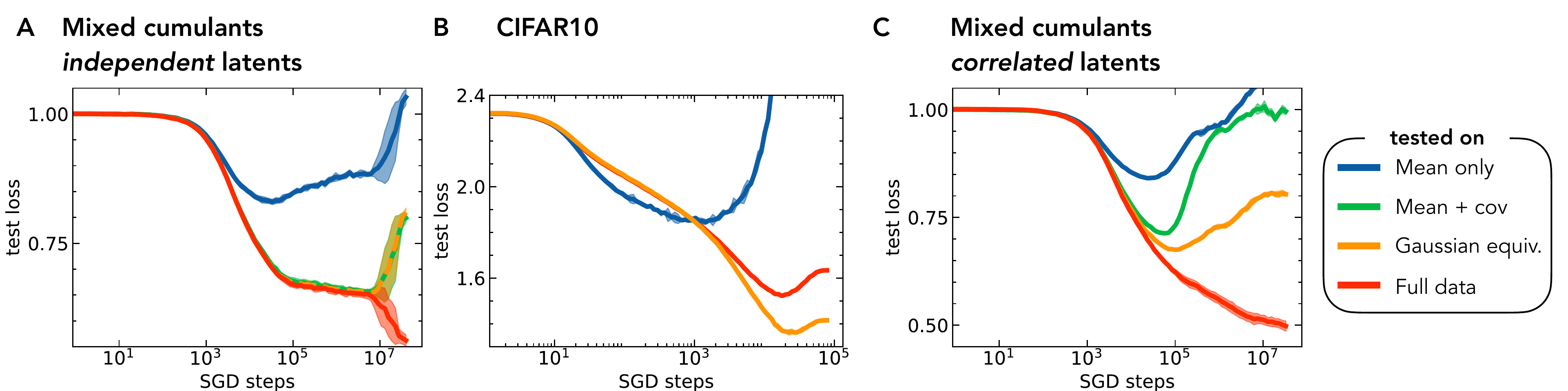}
  \caption{\label{fig:figure1} \textbf{Correlated latent variables speed up
      learning of neural networks.} \textbf{A} Test error of a two-layer neural
    network trained on the mixed cumulant model (MCM) of
    \cref{eq:mixed-cumulant-model} with signal-to-noise ratios
    $\beta_m=1, \beta_u=5, \beta_v = 10$. The MCM is a binary classification
    tasks where the inputs in the two classes have a different mean, a different
    covariance, and different higher-order cumulants (HOCs). We show the test
    error on the full data set (red) and on several ``censored'' data sets: a
    test set where only the mean of the inputs is different in each class (blue,
    $\beta_m=1, \beta_u=\beta_v=0$), a test set where mean and covariance are
    different (green, $\beta_m=1, \beta_u = 5, \beta_v=0$), and a Gaussian
    mixture that is fitted to the true data set (orange). The neural networks
    learn distributions of increasing complexity: initially, only the difference
    means matter, as the blue and red curves coincide; later, the network learns
    about differences at the level of the covariance, and finally at the level
    of higher-order cumulants. \textbf{B} Test loss of a two-layer neural
    network trained on CIFAR10 and evaluated on CIFAR10 (red), a Gaussian
    mixture with the means fitted to CIFAR10 (blue) and a Gaussian mixture with
    the means and covariance fitted on CIFAR10 (orange). \textbf{C}~Same setup
    as in \textbf{A}, but here the latent variables corresponding to the
    covariance and the cumulants of the inputs are correlated, leading to a
    significant speed-up of learning from HOCs (the red and orange line separate
    after $\gtrsim 10^{4}$ steps, rather than $\gtrsim 10^6$
    steps). \emph{Parameters:} $\beta_m=1, \beta_u=5, \beta_v=10, d=128, m=512$
    hidden neurons, ReLU activation function. Full details in
    \cref{app:figure-details}.}
\end{figure*}

\section{Results and insights: an informal overview}%
\label{sec:mixed-cumulant-model}

\subsection{The mixed cumulant model (MCM)}

We illustrate how correlated latent variables speed up learning by introducing a
simple model for the data, the mixed-cumulant model (MCM). The MCM is a binary
discrimination task where the signals that differentiate the two
classes are carried by different input cumulants. Specifically, we draw
$n$ data points $x^\mu=(x^\mu_i)\in\reals^d$ with $\mu=0, 1, \ldots, n$ either
from the isotropic Gaussian distribution~$\mathbb{Q}_{0}$ (label $y^\mu = -1$) or
from a \emph{planted distribution}~$\mathbb{Q}_{\text{plant}}$ (label $y^\mu=1$). Under
$\mathbb{Q}_{0}$, inputs have zero mean, an isotropic covariance equal to the
identity~$\id$, and all HOCs are zero. Under $\mathbb{Q}_{\text{plant}}$, the first few
cumulants each carry a signal that distinguishes the inputs from those in
$\mathbb{Q}_{0}$: a non-zero mean $m \in \mathbb{R}^d$, a covariance which is
isotropic except in the direction $u \in \mathbb{R}^d$, and higher-order
cumulants proportional to $v^{\otimes k}, k \ge 2$.  

We sample an input from
$\mathbb{Q}_{\text{plant}}$ by first drawing an i.i.d.\ Gaussian vector $z^\mu$ and two scalar
latent variables, the normally distributed $\lambda^\mu\sim \mathcal{N}(0, 1)$
and $\nu^\mu$, which is drawn from a non-Gaussian distribution. For concreteness, we will assume that
$\nu^\mu = \pm 1$ with equal probability. Both latent variables are independent
of $z$. Then, for $y^\mu=1$, we have
\begin{equation}
  \label{eq:mixed-cumulant-model}
  x^\mu = \underbrace{\overbrace{\beta_m m}^{\text{\textcolor{C0}{mean}}} +
  \overbrace{\sqrt{\beta_u} \lambda^\mu u}^{\text{covariance}}}_{\textcolor{C2}{\text{Gaussian}}}+
  \underbrace{S(\sqrt{\beta_v} \nu^\mu v}_{\text{\textcolor{C1}{HOCs}}} + z^\mu).
\end{equation}
where $\beta_i \ge 0$ are the signal-to-noise ratios associated to the three
directions, or ``spikes'', $m,u,v$. The spikes are fixed and drawn uniformly from the unit
sphere. We will sometimes force them to be orthogonal to one another. If $\beta_v=0$, it is easy to verify that inputs are Gaussian
with mean $\beta_m m$ and covariance $\id + \beta_u u u^\top$. If
$\beta_v>0$, inputs are non-Gaussian but the presence of the whitening matrix
\begin{equation}
  \label{eq:whitening}
  S = \id - \frac{\beta_v}{1 + \beta_v + \sqrt{1 + \beta_v}} v v^\top
\end{equation}
removes the direction $v$ from the covariance matrix, so that $v$ cannot be recovered from the input covariance if the latent variables $\lambda^\nu, \nu^\mu$ are uncorrelated.

\subsection{Neural networks take a long time to learn the cumulant spike in the vanilla MCM}

We show the performance of a two-layer neural network trained on the MCM model in \cref{fig:figure1}A with signal-to-noise ratios
$\beta_m=1, \beta_u=5, \beta_v=10$ and independent latent variables
$\lambda^\mu \sim \mathcal{N}(0, 1)$ and $\nu^\mu = \pm 1$ with equal
probability (red line). We can evaluate which of the three directions
have been learnt by the network at any point in time by \emph{evaluating} the
same network on a reduced test set where only a subset of the spikes are
present. Testing the network on a test set where the only
difference between the two classes $\mathbb{Q}_{0}$ and $\mathbb{Q}_{\text{plant}}$ are the mean of
the inputs (\textcolor{C0}{blue}, $\beta_u = \beta_v = 0$) or the \textcolor{C2}{mean and
  covariance} ($\beta_v=0$), we find that
two-layer networks learn about the different directions in a sequential way, learning first about the mean, then the covariance, and finally the higher-order cumulants. This is an example of the distributional simplicity bias~\citep{ingrosso2022data, refinetti2023neural,
  nestler2023statistical, belrose2024neural}. However, note that the periods of sudden improvement,
where the networks discovers a new direction in the data, are followed by long
plateaus where the test error does not improve. The plateaus notably delay the
learning of the direction $v$ that is carried by the higher-order
cumulants. 
We also see an
``overfitting'' on the censored test sets: as the network discovers a direction
that is not present in the censored data sets, this appears as overfitting on
the censored tests.

\subsection{Correlated latent variables speed up learning non-Gaussian
  directions}

As it stands, the MCM with orthogonal spikes and independent
latent variables is a poor model for real data: neural
networks show long plateaus between learning the Gaussian and non-Gaussian part
of the data on the MCM, but if we train the same network on CIFAR10, we see a smooth decay of the test loss, see \cref{fig:figure1}B. Moreover, testing a network trained on CIFAR10 on a Gaussian
mixture with the means fitted to CIFAR10 (\textcolor{C0}{blue}) and a Gaussian mixture with the means and covariance
fitted on CIFAR10 (orange) shows that the network goes smoothly and quickly from Gaussian to the non-Gaussian part of the data, without a plateau.

A natural idea improve the MCM, i.e.\ to make the loss curves of a neural network trained on the MCM resemble more the dynamics observed on CIFAR10, is to
correlate the covariance and cumulant spikes: instead of choosing them to be orthogonal to each other,
one could give them a finite overlap $u \cdot v = \rho$. However, a detailed analysis of the SGD dynamics in \cref{sec:two-directions} will show that this does \emph{not} speed up learning the
non-Gaussian part of the data; 
instead, the crucial ingredient to speed up learning of non-Gaussian correlations is the correlation
between latent variables. Setting for example
\begin{equation}
  \nu^\mu=\text{sign}(\lambda^\mu),
\end{equation}
we obtain the generalisation dynamics shown in red in \cref{fig:figure1}C: the
plateaus have disappeared, and we get a behaviour that is very close to real
data: a single exponential decay of the test loss.

\subsection{A rigorous analysis of a single neuron quantifies the speed-up of correlated latents}

We can make our experimental observations for two-layer networks rigorous in the simplest model of a neural
network 
a single neuron $f(w,x)=\sigma(w\cdot x)$ trained using online projected stochastic gradient descent (also known as the spherical
perceptron). At each step $t$ of the algorithm, we sample a tuple~$(x_t,y_t)$ from the MCM~\eqref{eq:mixed-cumulant-model} and update the weight $w_t$ according to the following rule:
\begin{equation}\label{eq:onlineSGD}
 w_t=   
\begin{cases}
w_0 \sim \text{Unif}\left(\mathbb S^{d-1}\right) &t=0 \\
\tilde w_t=w_{t-1}-\frac{\delta}d{\nablasph\left(\mathcal L(w,(x_t,y_t)\right)}  &t\ge 1 \\
w_t=\frac{\tilde w_t}{||\tilde w_t||}.
\end{cases}
\end{equation}
Here, $\nablasph$ is the spherical gradient defined by $\nablasph f(w)=(\id -ww^\top)\nabla f(w)$.  We follow \citet{damian2023smoothing} in using the \emph{correlation loss} for our analysis,
\begin{equation} \label{eq:corrloss}
\mathcal L(w,(x,y))=1-yf(w,x).
\end{equation}
In high dimensions, the typical overlap of the weight vector at initialisation and, say, the cumulant spike scales as $w_0 \cdot v \simeq d^{-1/2}$. We will say we the perceptron has learnt the direction $v$ if we have ``weakly recovered'' the spike, i.e.\ when the overlap $\alpha_v \equiv w \cdot v ~\sim O(1)$. This transition from diminishing to macroscopic overlap marks the exit of the search phase of stochastic gradient descent, and it often requires most of the runtime of online SGD~\citep{benarous2021online}. 

In this setup, we can give a precise characterisation of the sample complexity of learning a single spike (either in the covariance or in the cumulant) and of learning in the presence of two spikes with independent or correlated latent variables. By looking at stochastic gradient descent, we can also make quantitative statements on the optimal learning rates to achieve weak recovery quickly. For simplicity, our perceptron analysis does not  consider the spike in the mean, i.e. $\beta_m=0$ throughout. This assumption is mainly to enhance the mathematical tractability of the model and we expect most of the following to hold also in the case $\beta_m\ne 0$. Our main theoretical results are then as follows:

\paragraph{Learning a single direction:} If only the covariance \emph{or} the cumulant are spiked, i.e.\ either $\beta_u = 0$ or $\beta_v = 0$, the analysis of \citet{benarous2021online} applies directly and we find that projected online SGD requires $n \gtrsim d \log^2 d$ samples to learn the covariance spike $u$, but $n\gtrsim d^3$ samples to learn the cumulant spike; see \cref{prop:spiked_wish} and \cref{prop:spike_cumulant} for the precise scaling of the learning rate $\delta$. This result establishes that in isolation, learning from the higher-order cumulant has a much higher sample complexity than learning from the covariance.

\paragraph{Independent latent variables:} In the mixed cumulant model with spikes in the covariance and the HOCs, \cref{prop:negative_result} shows that if $n\le d^3$, it is impossible to have weak recovery of the cumulant spike for any learning rate $\delta_d=o\left(\nicefrac{1}{d}\right)$ \emph{if} the latent variables $\lambda$ and $\nu$ are independent -- hence learning from HOCs remains hard. For larger learning rates $\nicefrac 1d\le \delta_d=o(1)$, the SGD noise becomes dominant after $\nicefrac {d}{\delta^2}$ steps and our analysis works only up to that horizon; we discuss this  in \cref{sec:two-directions}.

\paragraph{Correlated latent variables:} If instead both spikes are present and their latent variables have a positive correlation $\E[\lambda^\mu\nu^\mu]>0$ (fixed, independent of $d$),  proposition \ref{prop:positive_result} shows that $n \gtrsim d \log^2 d$ samples are \emph{sufficient to weakly recover both spikes} with the optimal learning rate $\delta\approx\nicefrac{1}{\log d}$. Moreover, even for sub-optimal learning rates $\delta=o\left(\nicefrac{1}{d}\right)$, the time to reach weak recovery of the cumulant spike is $d^2\log^2 d$, which is still faster that in the uncorrelated case.
\Cref{prop:positive_result} also shows that the speed-up in terms of sample complexity required for weak recovery happens even when the amount of correlation is small compared to the signal carried by the cumulant spike. In other words, it does not affect the minimum landscape of the loss: if the global minimum is close to the cumulant spike, introducing correlations between the latents will not move that minimum to the covariance spike.

We give a full summary of the known results and our contributions in \cref{tab:results}. In the following, we give precise statements of our theorems in \cref{sec:odes} and discuss our results in the context of the wider literature, and the recently discussed ``staircase phenonmenon'', in \cref{sec:discussion}.

\section{Rigorous analysis of the perceptron: the non-Gaussian information exponent}%
\label{sec:odes}

We now present a detailed analysis of the simplest model where correlated latents speed up learning from higher-order cumulants in the data, the spherical perceptron. Before stating the precise theorems, we present the main ingredients of the analysis informally. 

The key idea of our analysis is borrowed from \citet{benarous2021online}: during the \emph{search phase} of SGD, while the overlaps $\alpha_{u}=u\cdot w,\alpha_v=v\cdot w $ of the weight $w$ with the covariance and cumulant spikes are small, i.e.\ $o(1)$ with respect to $d$, the dynamics of spherical SGD is driven by the low-order terms of the polynomial expansion of the \emph{population loss} $\mathcal L(w)=\E[\mathcal L(w,(x,y)]$. In our case of a mixture classification task, it is useful to rewrite the expectation using the \emph{likelihood ratio} between the isotropic and the planted distributions $L(x):=\nicefrac{\text{d}\Q_{\text{plant}}}{\text{d}\Q_0}(x)$ such that all averages are taken with respect to the simple, isotropic Gaussian distribution:
\begin{align}
\mathcal L(w)&=1+\frac12\E_{\Q_0}[\sigma(w\cdot x)]-\frac12 \E_{\Q_0}\left[L(x) \sigma(w\cdot x)\right].
\end{align}
We can then expand the loss in Hermite polynomials, which form an orthonormal basis w.r.t.\ the standard Gaussian distribution, and prove that the loss depends only on $\alpha_u,\alpha_v$ and can be expanded in the following way:
\begin{equation} \label{eq:popu_loss_expansion}
    \mathcal L(w)=\ell(\alpha_{u},\alpha_v)=\sum_{i,j=0}^\infty c^L_{ij}c^\sigma_{i+j}\alpha_u^i\alpha_v^j,
\end{equation} 
see \cref{lemma:populoss_formula} in the appendix.
The degree of the lowest order term in this expansion is called the \emph{information exponent}~$k$ of the loss and it rigorously determines the duration of the search phase before SGD \emph{weakly recovers} the key directions~\citep{benarous2021online, dandi2023learning}, in the sense that for any time-dependent overlap $\alpha(t)$, there exists an $\eta>0$ such that defining $\tau_{\eta}:=\min\left\{t\ge0 \big| \ |\alpha(t)|\ge \eta\right\}$, we have
 \begin{equation}
 \label{eq:weak-recovery}
     \lim_{d\to \infty} \mathbb P\left(\tau_\eta\le n\right)=1.
 \end{equation}

Since the time in online SGD is equivalent to the number of samples, the information exponent governs the sample complexity for weak recovery. Here, we are interested in finding out how the correlation of the latent variables changes the information exponent, and which consequences this change has on the duration of search phase of projected SGD as $d\to \infty$.
Considering only the terms that give relevant contribution for the search phase, \cref{eq:popu_loss_expansion} applied to the MCM model with $\beta_m=0$ becomes (see \cref{app:formulas} for more details):
\begin{equation} \label{eq:popu_loss_MCM}
    \ell(\alpha_{u},\alpha_v)=-\left(c_{20}\alpha_u^2+c_{11}\alpha_u\alpha_v+c_{04}\alpha_v^4\right)
\end{equation} 
where $c_{20},c_{04}>0$ as long as $\beta_u,\beta_v>0$, whereas $c_{11}>0$ if and only if the latent variables are positively correlated: $\mathbb E[\lambda \nu]>0$.
Note that switching on this correlation does not change the overall information exponent, which is still 2,  but the \emph{mixed term} $\alpha_u\alpha_v$ strongly impacts the direction $v$, along which $\partial_v\ell=c_{11}\alpha_u+4c_{04}\alpha_v^3$  has degree $3$ if $c_{11}=0$ whereas it has degree 1 in case of positive correlation. This means that \emph{positive correlation of latent variables lowers the information exponent along the non-Gaussian direction}.

It is not straightforward to link these changes of the population-loss series expansion to actual changes in SGD dynamics. The randomness of each sample comes into play and it must be estimated thanks to a careful selection of the learning rate $\delta$ (see \cref{prop:negative_result} and \cref{prop:positive_result}). However, at the risk of oversimplification, we could imagine that the dynamic is updated descending the spherical gradient of the population loss, leading to the following  system of ODEs that approximately hold when $\alpha_u,\alpha_v\le \eta$  
\begin{equation} \label{eq:ODEsystMCM}
\begin{cases}
    \dot \alpha_{u}(t)= 2c_{20}\alpha_u+c_{11}\alpha_v+O(\eta^2)\\
     \dot \alpha_{v}(t)= c_{11}\alpha_u+4c_{04}\alpha_{v}^3-2c_{20}\alpha^2_u\alpha_v +O(\eta^4)\\
\end{cases}
\end{equation}
where the last term in the second equation is due to the distorting effect of the spherical gradient.

We can see that the behaviour of $\alpha_u$ is not affected too much by the presence of correlation (the integral lines starting at $\alpha_u(0)\approx d^{-1/2}$ reach order 1 in $t=d\log^2d$). On the contrary the dynamic of $\alpha_v$ changes completely: if $c_{11}>0$, $\alpha_v$ will grow at the same pace as $\alpha_u$, otherwise it will be orders of magnitude slower (there is even the possibility that it will never escape 0, since the term $-2c_{20}\alpha^2_u\alpha_v$ could push it back too fast). Before stating the theorems that make these intuitions precise, we have to state our main assumptions on the activation function:
\begin{assumption}[Assumption on the student activation function $\sigma$]%
\label{assumptions:sigma} We require that the activation function $\sigma\in\mathscr C^1(\R)$ fulfils the following conditions:
\begin{equation}\label{eq:sigma hp}
\begin{aligned}
    \underset{z\sim\mathcal N(0,1)}{\mathbb E}\left[\sigma(z)h_2(z)\right]>0,&
   \underset{z\sim\mathcal N(0,1)}{\mathbb E}\left[\sigma(z)h_4(z)\right]<0,\\
         \sup_{w\in \mathbb S^{d-1}} \E\left[\sigma'(w\cdot x)^{4}\right]\le C,&
            \sup_{w\in \mathbb S^{d-1}} \E\left[\sigma'(w\cdot x)^{8+\iota}\right]\le C,\\
            \mathrm{and} \sum_{k=0}k\underset{z\sim \mathcal N(0,1)}{\E}[h_k(z)\sigma(z)]&<\infty.
\end{aligned}
\end{equation}
Here, $(h_i)_i$ are normalised Hermite polynomials (see \cref{app:hermite}) and $\iota>0$. 
\end{assumption}
Note that smoothed versions of the ReLU activation function satisfy these conditions. We remark that all \cref{eq:sigma hp} are important requirements, activation functions that do not satisfy one or more of these conditions could lead to different dynamics. On the other hand  the assumption $\sigma \in \mathscr C^1(\R)$ is done for convenience and differentiability a.e.\ should be enough; for instance all the simulations of the present work were done with the ReLU activation function.

\subsection{Learning a single direction}%
\label{sec:one-direction}

\begin{table*}[t!]
\renewcommand{\arraystretch}{1.2}
\centering
\begin{tabular}{ c c  c c}
\toprule
Model & Step size & Weak recovery $u$ & Weak recovery $v$ \\ \midrule
Covariance spike only &  \textcolor{orange}{$\delta=O(1)$},  \textcolor{blue}{$\quad\delta \approx\frac{1}{\log d} $} &  \textcolor{orange}{$n\ll d\log d$}, \textcolor{blue}{$n\gg d\log^2 d$} & not present \\ \midrule
Cumulants spike only & \textcolor{orange}{$\delta=o\left(\frac1d\right)$},\textcolor{blue}{$\quad\delta\approx \frac{1}{d} $} & not present & \color{orange}$n\ll d^3\quad$ \color{blue} $n\gg d^3\log^2 d$\color{black} \\ \midrule
 &  \textcolor{blue}{$\delta \approx \frac{1}{\log d} $}& \textcolor{blue}{$ n \gg d\log^2 d$} & \textcolor{orange}{$n\ll d \log^2 d$} \\ \cmidrule(l){2-4} 
\multirow{-2}{*}{\begin{tabular}[c]{@{}c@{}}Two spikes,\\ independent latents\end{tabular}} &   \textcolor{orange}{$\delta=o\left(\frac1d\right)$},\textcolor{blue}{\quad$\delta\approx \frac{1}{d} $} & \textcolor{blue}{$ n\gg d^2\log d$} & \textcolor{orange}{$n\ll d^3$} \\ \midrule
 &  \textcolor{blue}{$\delta \approx \frac{1}{\log d} $}& \textcolor{blue}{$ n \gg d\log^2 d$} & \textcolor{blue}{$ n \gg d\log^2 d$} \\ \cmidrule(l){2-4} 
\multirow{-2}{*}{\begin{tabular}[c]{@{}c@{}}Two spikes,\\ correlated latents\end{tabular}} &  \textcolor{blue}{$\delta\approx \frac{1}{d} $} & \textcolor{blue}{$ n\gg d^2\log d$} & \textcolor{blue}{$ n \gg d^2\log d$} \\ \bottomrule
\end{tabular}%
\caption{\label{tab:results} \textbf{Summary of positive and negative results on learning from higher-order cumulants.} We summarise the sample complexities required to reach weak recovery of the covariance spike $u$ and the cumulant spike $v$ in the different settings discussed in \cref{sec:odes}.  \textcolor{orange}{Negative results} mean that weak recovery is impossible below these sample complexities at the given learning rates.  \textcolor{blue}{Positive results} guarantee weak recovery with at least as many samples as stated. The first two lines are direct applications of results from \citet{benarous2021online} for data models with a single spike. The mixed cumulant model with independent latent variables is a direct corollary of these results. 
The notation $\delta\approx f(d)$ for the positive results means that $\delta =o(f(d))$, but the best sample complexity is achieved when $\delta$ is as close as possible to its bound. Note that for independent latents with large learning rate $\delta \approx \nicefrac{1}{\log d}$, our techniques only  allow studying a time horizons up to $t\approx d\log^2 d$, after which the randomness of the SGD updates becomes too large to be controlled, see \cref{sec:two-directions} for more details.}
\end{table*}

To establish the baselines for weak recovery in single index models, where inputs carry only a covariance spike or a spike in the higher-order cumulants, we apply the results of \citet{benarous2021online} directly to find the following SGD learning timescales:

\begin{proposition}[Covariance spike only] Considering only the spike in the covariance the generative distribution is $y\sim \text{Rademacher}\left(\nicefrac1{2}\right)$,
\begin{align*}
    y=1& \ \Rightarrow \ x^\mu=\sqrt \beta_u\lambda^\mu u +z^\mu, \quad& \lambda^\mu\sim \mathcal N(0,1)\\
     y=-1& \ \Rightarrow \ x^{\mu}= z^\mu, \quad& z^\mu \sim \mathcal N(0,\mathbbm 1_d)
\end{align*}
A spherical perceptron that satisfies \cref{assumptions:sigma} and is trained on the correlation loss~\eqref{eq:corrloss} using online SGD~\eqref{eq:onlineSGD} has the following result concerning the overlap with the hidden direction $\alpha_{u,t}:=u\cdot w_t$:
\begin{itemize}\item if $\frac{1}{\log^2 d}\ll\delta_d\ll\frac{1}{\log d}$ and 
$n\gg d\log^2 d$ then there is \textbf{strong recovery} and 
\[
|\alpha_{u,t}|\to 1 \qquad \text{in probability and in $\mathscr L_p$ }\forall p\ge 1
\]
\item if $\delta_d=O(1)$ and  $n \le d\log d$ then it is impossible to learn anything:
\[
|\alpha_{u,t}|\to 0 \qquad \text{in probability and in $\mathscr L_p$ }\forall p\ge 1
\]
\end{itemize}
\label{prop:spiked_wish}

\end{proposition}
\begin{proposition}[Cumulant spike only]%
\label{prop:spike_cumulant} Considering only the spike in the cumulants, letting $(y^\mu)_\mu,(\nu^{\mu})_\mu$ be i.i.d. Rademacher$(1/2)$, $(z^\mu)_\mu$ be i.i.d $N(0,\mathbbm 1_d) $ and $S$ as in \eqref{eq:whitening}, the generative distribution is:
\begin{align*}
    y^\mu=1& \ \Rightarrow \ x^\mu=S\left(\sqrt \beta_v\nu^\mu v +z^\mu\right)\\
     y^\mu=-1& \ \Rightarrow \ x^{\mu}= z^\mu,
\end{align*}
A spherical perceptron that satisfies \cref{assumptions:sigma} and is trained on the correlation loss~\eqref{eq:corrloss} using online SGD~\eqref{eq:onlineSGD} has the following result concerning the overlap with the hidden direction $\alpha_{v,t}:=u\cdot v_t$:
\begin{itemize}\item if $\frac{1}{d^2\log^2 d}\ll\delta_d\ll\frac{1}{d\log d}$ and 
$n\gg d^3\log^2 d$ then there is \textbf{weak recovery} of the cumulant spike in the sense of \cref{eq:weak-recovery};
\item if $\delta_d=o(\frac {1}{d} )$ and  $n\ll d^3$ then it is impossible to learn anything:
\[
|\alpha_{v,t}|\to 0 \qquad \text{in probability and in $\mathscr L_p$ }\forall p\ge 1
\]
\end{itemize}
\end{proposition}

Note that the thresholds found in \cref{prop:spike_cumulant} coincide with the ones found for tensor-PCA for a order 4 tensor (see proposition 2.8 in \citep{benarous2021online}), which is exactly the order of the first non trivial cumulant of $\Q_{\text{plant}}$ in the spiked-cumulant model. We will also see what happens in case of larger step size than what considered in \cref{prop:spike_cumulant}.

\subsection{Learning two directions}%
\label{sec:two-directions}

The following two propositions apply to a perceptron trained on the MCM model of \cref{eq:mixed-cumulant-model} with $\beta_m=0$, $\nu^\mu\sim \mathrm{Radem}(1/2)$. We first state a negative result: in the case of \emph{independent} latent variables, the cumulant direction cannot be learned faster than in the one dimensional case (\cref{prop:spike_cumulant}). 
\begin{proposition} \label{prop:negative_result}
Consider a spherical perceptron that satisfies \cref{assumptions:sigma} and is trained on the MCM~\eqref{eq:mixed-cumulant-model} with correlation loss~\eqref{eq:corrloss} using online SGD~\eqref{eq:onlineSGD}, with $\lambda^\mu$ independent of $\nu^\mu$. Let $\delta_d=o(1)$ be the step size. As long as $n\ll \min\left(\frac {d}{\delta_d^2},d^3\right)$, we have that
    \begin{equation*} \label{eq:negative_result}
    \lim_{d\to \infty} \sup_{t\le n} \left |\alpha_{v,t}\right|=0
    \end{equation*}
     where the limit holds in $\mathscr L^p$ for every $p\ge1$.
\end{proposition}
The condition $n\ll \nicefrac {d}{\delta_d^2}$ is due to \emph{sample complexity horizon} in online SGD.  As already pointed out by \citet{benarous2021online}, it turns out that after a certain number of steps, when $\nicefrac{n \delta_d^2}{d}$ becomes large, the noise term in the equations dominates the drift term and the dynamics starts to become completely random  (mathematically, Doob's inequality \eqref{eq:mart_estimfirst} fails to provide a useful estimate on the noise after this point).
So our results can prove the absence of weak recovery only up to the \emph{sample complexity horizon}; after that our definition of \emph{weak recovery} is not useful anymore: due to the increased randomness, it is possible that the weight vector attains, by pure chance, a finite overlap with the spiked directions, but that would be forgotten very easily, not leading to meaningful learning. 
If instead the learning rate is small enough, $\delta_d=o(\frac 1d)$, then the horizon becomes large enough to include $n\ll d^3$, which is the same regime of \cref{prop:spike_cumulant}.
 
Now we state the positive results: the covariance spike can be weakly recovered in the same sample complexity as the one-dimensional case (\cref{prop:spiked_wish}). However, if the latent variables have positive correlation, the same sample complexity is also sufficient to weakly recover the cumulant spike.
\begin{proposition} \label{prop:positive_result} In the setting described, let the total number of samples be $n=\theta_d d$, with $\theta_d\gtrsim \log^2 d$ and growing at most polynomially in $d$. The step size $(\delta_d)_{d}$ is chosen to satisfy:
 \begin{equation} \label{eq:hypothesis_on_delta}
 \frac{1}{\theta_d}\ll\delta_d\ll\frac{1}{\sqrt{\theta_d}}
\end{equation}
Projected SGD reaches \textbf{weak recovery} of the covariance spike in the sense of \cref{eq:weak-recovery} in a time $\tau_u \le n$. 
Moreover, if the latent variables have positive correlation:~$\mathbb E[\lambda \nu]>0$, conditioning on having matching signs at initialisation: $\alpha_{u,0}\alpha_{v,0}>0$, \textbf{weak recovery is reached also for the cumulant spike $v$ in a time $\tau_v \le n$}.
 \end{proposition}
The initialisation assumption $\alpha_{u,0}\alpha_{v,0}>0$ means that the second part of \cref{prop:positive_result} can be applied on half of the runs on average. This requirement could be fundamental: in case of mismatched initialisation, the correlation of the latent variables would push both $\alpha_{u,t}$ and $\alpha_{v,t}$ towards 0, slowing down the process. Note that the dependence of SGD dynamics on initialisation is a recurring phenomenon that arises with complex tasks, it is for example a well-known effect when training a neural network on inputs on a XOR-like Gaussian mixture task~\citep{refinetti2021classifying, benarous2022high}. It is likely that this initialisation problem is specifically due to having a single neuron network. In a two-layer neural network, overparametrisation usually helps since it is sufficient that there is a neuron with the right initialisation signs to drive the learning in the right direction~\cite{refinetti2021classifying}. For instance, the simulations we performed for \cref{fig:figure1} with wide two-layer networks did not exhibit initialisation problems.

 \begin{remark}\label{remark:step_size} Optimising in $\theta$ in \cref{prop:positive_result}, we get that if $\frac{ 1}{\log^2 d}\ll\delta_d\ll \frac{1}{\log d}$, then $n\gg d\log^2 d $ guarantees to have weak recovery.   
         This optimal scaling of $\delta_d$ does not coincide with the best learning rate scaling for \cref{prop:negative_result}, since the sample complexity horizon is very short.  However, by taking the slightly sub-optimal learning rate $\delta_d=\frac{1}{d\log d}$, we get that on one hand \cref{prop:negative_result} applies and we can infer that, with independent latent variables, it takes at least~$d^3$ samples/steps to learn direction $v$. On the other hand, by \cref{prop:positive_result} we know that when the latent variables have (even a small) correlation, in $n\approx d^2\log^2 d$ projected SGD attains weak recovery of both spikes.
 \end{remark}

This completes our analysis of the search phase of SGD, which established rigorously how correlated latent variables speed up the weak recovery of the cumulant spike. We finally note that this analysis can only be applied to the \emph{search phase} of SGD. The expansion in \cref{eq:popu_loss_expansion} breaks down as soon as $\alpha_u, \alpha_v$ become macroscopic, since there are infinitely many terms that are not negligible along the non-Gaussian direction $v$ in that case. Hence, even though the correlation between latent variables is key in the early stages of learning, it can become a sub-leading contribution in the later stages of SGD where the direction of weight updates is greatly influenced by terms due to HOCs. A thorough analysis of the descent phase after weak recovery will require different techniques: the SGD dynamical system involving the complete expression of the population loss should be studied, 
we leave this for future work. 

However, the following simple example highlights the richness of the dynamics in this regime. Consider a sequence $(\beta_u^m,\beta_v^m)$ such that $(\beta_u^0,\beta_v^0)=(1,0)$ and $\lim_{m\to \infty}(\beta_u^m,\beta_v^m)=(0,1)$. The global minimum of the population loss will move continuously from $w=u$ to~$w=v$. However, looking at the dynamics of \cref{eq:ODEsystMCM}, the change will be only at the level of the coefficients, and as long as none of them vanishes, the dynamics will be qualitatively the same. So changing the signal-to-noise ratios in this regime will affect only the the descent phase of the problem.
 
\section{Discussion: hierarchies of learning with neural networks}%
\label{sec:discussion}

\subsection{Learning functions of increasing complexity and the staircase phenomenon}

There are several well-known hierarchies that characterise learning in neural
networks. In a seminal contribution, \citet{saad1995exact, saad1995online}
characterised a ``specialisation'' transition where two-layer neural networks 
with a few hidden neurons go from performing like an effective (generalised)
linear model to a non-linear model during training with online SGD. Notably, this transition occurs after a long plateau in the generalisation
error. A similar learning of functions of increasing complexity was shown
experimentally in convolutional neural networks by~\citet{kalimeris2019sgd}. 

More recently, \citet{abbe2021staircase, abbe2022merged} studied
the problem of learning a target function over binary inputs that depend only on
a small number of coordinates. They showed that two-layer neural networks can
learn $k$-sparse functions (i.e.\ those functions that depend only on $k$
coordinates) with sample complexity $n \gtrsim d$, rather than the
$n \gtrsim d^k$ sample complexity of linear methods operating on a fixed feature
space, if the target functions fulfil a (merged) \emph{staircase}
property. \citet{abbe2023sgd} extended this analysis to study the saddle-to-saddle dynamics of two-layer neural networks for specific target link functions. Similar saddle-to-saddle dynamics have been described by \citet{jacot2021saddle} in linear networks and by \citet{boursier2022gradient} for two-layer ReLU networks.

A detailed description of the staircase phenomenon when learning a multi-index
target function over Gaussian inputs was given by
\citet{dandi2023learning}, who studied feature learning via a few steps of single-pass, large-batch, gradient descent\footnote{Note that a single step of gradient descent was studied in a similar setting by
  \citet{ba2022high, damian2022neural} without analysing the
  staircase phenomenon.}. Their analysis showed how subsequent steps of gradient descent allow for learning new perpendicular directions of the target function if those directions
are linearly connected to the previously learned directions. \citet{bietti2023learning} also consider multi-index target function
over Gaussian inputs, but instead look at the gradient flow for two-layer neural networks and provide a complete picture of the timescales of the ensuing saddle-to-saddle dynamics. \citet{berthier2023learning} analyse the
gradient flow of two-layer neural networks when both layers are trained
simultaneously. They perform a perturbative expansion in the (small) learning
rate of the second layer and find the timescales over which the neural network learns the Hermite coefficients of the target function sequentially. 

\subsubsection{Relation between the mixed cumulant and teacher-student models}%
\label{sec:related-teacher-student}

\begin{figure*}[t!]
  \centering
  \includegraphics[trim=0 250 0 460,clip,width=\linewidth]{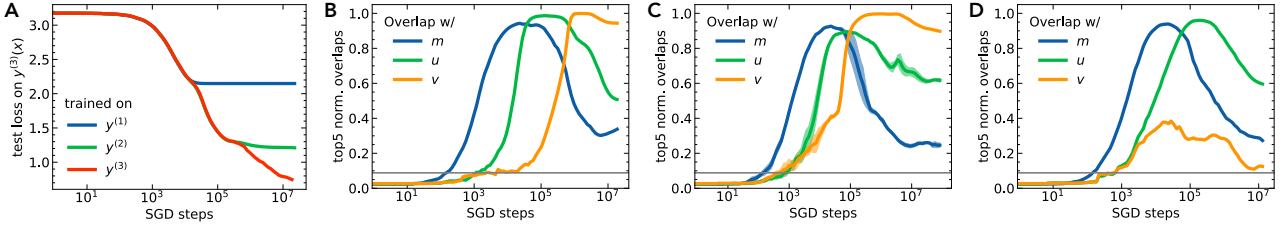}
  \caption{\label{fig:teacher-student} \textbf{Staircases in
      the teacher-student setup.} \textbf{A} Test accuracy of the same two-layer neural networks as in \cref{fig:figure1} evaluated on the degree-4 target function $y^*(x)$~\eqref{eq:teacher} during training on the target functions $y^{(1)}(x) = h_1(m \cdot x )$ (blue), $y^{(2)}(x) = h_1(m \cdot x) + h_2(u \cdot x)$ (green), and the teacher function \cref{eq:teacher} (red). Inputs are drawn from the standard multivariate Gaussian distribution. \textbf{B-D} We show the the average of the top-5 largest normalised overlaps $w_k \cdot u $ of the weights of the $k$th hidden neuron $w_k$ and the three directions that need to be learnt for three different target functions: the teacher function in \cref{eq:teacher} (\textbf{B}), the same teacher with inputs that have a covariance $\id + u v^\top+vu^{\top}$ (\textbf{C}), and a teacher with mixed terms, \cref{eq:teacher-mixed} \textbf{(D)}. The dashed black line is at $d^{-1/2}$, the threshold for weak recovery. \emph{Parameters:} Simulation parameters as in \cref{fig:figure1}: $d=128, m=512$ hidden neurons, ReLU activation function. Full details in \cref{app:figure-details}.}
\end{figure*}

Given these results, it is natural to ask how the staircases in teacher-student
setups relate to the hierarchical learning in the mixed-cumulant model. A
teacher model with three ``spikes'' (or teacher weight vectors) where each spike
is learnt with a different sample complexity, analoguously to the uncorrelated MCM, is the function
\begin{equation}
  \label{eq:teacher}
  y^*(x) = h_1( m\cdot x ) + h_2( u\cdot x
  ) + h_4(v\cdot x),
\end{equation}
where $h_k$ is the $k$th Hermite polynomial (see \cref{app:hermite} for details) and the inputs
$x$ are drawn i.i.d.\ from the standard Gaussian distribution with identity
covariance. In the language of \citet{benarous2021online}, a single-index model
with activation $h_k$ has information exponent $k$, so taken individually, the three target
functions in \cref{eq:teacher} with directions $m, u$ and $v$ would be learnt with $n \gtrsim d, d \log^2 d$, and~$d^3$ samples by a spherical perceptron using online SGD, exactly as in our mixed cumulant model with independent latent variables. We show the sequential learning of 
the different directions of such a teacher for a two-layer neural network in \cref{fig:teacher-student}A. In particular, we plot the test error with respect to the teacher $y^*(x)$~\eqref{eq:teacher} for a two-layer network trained directly on the teacher (red), and for two-layer networks trained on only the first / the first two Hermite polynomials of \cref{fig:teacher-student} (blue and green, respectively). Their performance suggests that the three directions $m, v, u$ are learnt sequentially, which is further supported by showing the maximum normalised overlap of the hidden neurons with the different spikes shown in \cref{fig:teacher-student}B. 

How can we speed up
learning of the direction $u$ with information exponent $4$ in a way that is analogous to the correlated latents? The latent variables of the MCM correspond to the pre-activations of the target function in the sense that both are low-dimensional projections of the inputs that determine the label. In the standard teacher-student setup, where inputs are drawn i.i.d.\ from a multivariate normal distribution with identity covariance, and target functions are of two-layer networks of the form $y^*(x)=\sum_k v_k \sigma_k ( u_k\cdot x )$, it is not possible to have correlated pre-activations. Instead, we can speed up learning of the direction $v$ by introducing a mixed term in the target function,
\begin{equation}
   \label{eq:teacher-mixed}
   y^*(x) = h_1( m\cdot x) +  h_1( u\cdot
  x )  h_1( v\cdot x ) + h_2( u\cdot x
  ) + h_4( v\cdot x ),
\end{equation}
as is apparent from the overlap plot in \cref{fig:teacher-student}C. \citet{dandi2023learning} and \citet{bietti2021deep} discuss the speed-up of learning this type of target function in the one-step and gradient flow frameworks, respectively, whereas our analysis focuses on online stochastic gradient descent with non-Gaussian inputs. An alternative for correlating the pre-activations is to  keep the target function \cref{eq:teacher} while drawing inputs from a normal distribution with covariance $\id + uv^\top$, and we see in \cref{fig:teacher-student}D that it does indeed speed up learning of the cumulant direction. This result is similar to the setting of \citet{mousavi2024gradient}, who found that ``spiking'' of the input covariance can speed up learning a single-index model with a two-layer network.

So in summary, we see that the accelerated learning requires fine-tuning between the target function and the input structure. The mixed cumulant model does not require an explicit target function and instead directly highlights the importance of the correlation in the latent variables of the inputs to distinguish different classes of inputs. 

\subsection{The role of Gaussian fluctuations with correlated latents}%
\label{sec:gaussian-equivalent-model}

Correlation between latent variables changes the covariance of the inputs compared to independent latents. A
quick calculation shows that the covariance of the inputs under the planted distribution $C_{uv} \equiv \underset{\mathbb{Q}_{\text{plant}}}{\mathrm{cov}}(x, x) $ becomes
\begin{equation}
  \label{eq:cov}
  C_{uv} = \id + \beta_u u u^\top + \sqrt{\frac{\beta_u
    \beta_v}{1+\beta_v} }\;  \EE \lambda \nu \left( u v^\top + v u^\top \right).
\end{equation}
In other words, we can weakly recover the cumulant spike by computing the leading eigenvector of the covariance of the inputs. Does that mean that the speed-up in learning due to correlated latents is simply due to the amount of information about the
cumulant spike $v$ that is exposed in the covariance matrix? 

We test this hypothesis by taking the two-layer neural networks trained on the full MCM model and evaluating them on
an equivalent Gaussian model, where we replace inputs from the planted distribution $\mathbb{Q}_{\text{plant}}$ of the MCM with samples from a
multivariate normal distribution with mean~$\beta_m u_m$ and covariance~$C_{uv}$. We plot the corresponding test losses for independent and correlated latents in orange in \cref{fig:figure1}. When latent variables are independent, the Gaussian approximation breaks down after after a long plateau at $\approx 10^7$ steps, precisely when the network discovers the non-Gaussian fluctuations due to the cumulant spike. For correlated latent variables on the other hand, the Gaussian approximation of the data only holds for $\approx 10^4$ steps. This suggests that the presence of the cumulant spike in
the covariance gives some of the neurons a finite overlap
$w_k \cdot v$ with the cumulant spike, and this initial overlap speeds up the recovery of the cumulant spike using information from the higher-order cumulants of the data which is inaccessible at those time scales when latent variables are uncorrelated. The presence of correlated latent variables therefore genuinely accelerates the learning from higher-order, non-Gaussian correlations in the inputs.

\subsection{Further related work}%
\label{sec:further-related-work}

\paragraph{Learning polynomials of increasing degree with kernel methods} A
series of works analysing the performance of kernel
methods~\citep{dietrich1999statistical, ghorbani2019limitations,
  ghorbani2020neural, bordelon2020spectrum, spigler2020asymptotic,
  xiao2022precise, cui2023optimal} revealed that kernels (or linearised neural
networks in the ``lazy'' regime~\cite{li2017convergence, jacot2018neural,
  arora2019exact, li2018learning, chizat2019lazy}) require~$n \gtrsim d^\ell$
samples to learn the $\ell$th Hermite polynomial approximation of the target
function. Here, we focus instead on the feature learning regime where neural networks can access features from higher-order cumulants much faster.

\paragraph{Tensor PCA with side information} From the perspective of unsupervised learning, the MCM model is closely related to models studied in random matrix theory. When $\beta_m=\beta_\nu=0$, the data distribution boils down to the well known Spiked Wishart model, that exhibits the \emph{BBP phase transition} \cite{baik2004-bbp}, which predicts thresholds on $\beta_c$ for detection, at linear sample complexity $n\approx d$. If instead only $\beta_\nu\ne 0$, the MCM model corresponds to a tensor PCA problem, where one seeks to recover a signal
$\xi \in \reals^d$ from a noisy order-$p$ tensor $T=\xi^{\otimes p} + \Delta$ with suitable noise tensor $\Delta$ (that in most of the cases is assumed to be Gaussian distributed). Indeed, in the MCM model the HOCs spike $v$ could be retrieved by performing tensor PCA on the empirical kurtosis tensor. 

The notion of correlated latent variables that we consider in the present work is reminiscent to the concept of \emph{side information} that \citet{richard2014statistical} considered for tensor PCA.  The side information would be an additional source of information on the spike via a Gaussian channel, $y = \gamma \xi + g$, where $g$ is a Gaussian
noise vector and $\gamma>0$; using this estimate as an initialisation for AMP
leads to a huge improvement in estimation. This joint model can be seen as a rank-1 version
of a topic modelling method analysed by \citet{anandkumar2014tensor}. In a similar vein, \citet{mannelli2020marvels} introduced the spiked matrix-tensor
model, in which the statistician tries to recover $\xi$ via the observation of a
spiked matrix $M \propto \xi \xi^\top + \Delta^{(1)}$ and an order-$p$ tensor
$T \propto \xi^{\otimes p} + \Delta^{(2)}$ with appropriately scaled noise
matrix / tensor $\Delta^{(1)}$ and $\Delta^{(2)}$, respectively. They analyse
the optimisation landscape of the problem and the performance of the Langevin
algorithm~\citep{mannelli2020marvels} and of gradient
descent~\cite{mannelli2019passed, sarao2019afraid}. The main difference to our
model is that these works consider recovering a \emph{single} direction that
spikes both the matrix and the tensor; we consider the case where two orthogonal
directions are encoded as principal components of the covariance matrix and the
higher-order cumulant tensors of the data.

\section{Concluding perspectives}

To achieve good performance, neural networks need to unwrap the higher-order
correlations of their training data to extract features that are pertinent
for a given task. We have shown that neural networks exploit correlations
between the latent variables corresponding to these directions to speed up
learning. In particular, our analysis of the spherical perceptron showed that
correlations between the latent variables corresponding to the directions of two
cumulants of order $p$ and $q > p$, respectively, will speed up the learning of
the direction from the $q$th cumulant by lowering its information
exponent.

Our results open up several research directions. First, it will be intriguing to
extend our analysis from single-index to shallow models. For two-layer neural
networks, a key difficulty in applying mean-field techniques to analyse the
impact of data structure on the dynamics of learning is the breakdown of the Gaussian equivalence principle~\cite{mei2022generalization, goldt2020modeling, goldt2022gaussian, gerace2020generalisation, hu2022universality, pesce2023gaussian} that we discussed in \cref{sec:gaussian-equivalent-model}. 
In the presence of non-Gaussian pre-activations $w_k \cdot x$, it is not immediately
clear what the right order parameters are that completely capture the dynamics. 
An analysis of the dynamics in the one-step framework of \citet{ba2022high, dandi2023learning} is another promising approach to tackle the dynamics of two-layer networks. Other intriguing directions include determining the hardness of learning from higher-order cumulants when re-using data~\citep{dandi2024benefits} and from the perspective of generative exponents~\citep{damian2024computational}. For
deeper neural networks, it is important to investigate the role of different
layers in processing the HOCs of the data (even numerically), for example along
the lines of
\citet{fischer2022decomposing}. Finally, it will be intriguing to investigate how
correlated latent variables emerge both in real data and in (hierarchical) models of synthetic data.

%% file: content/acknowledgements.tex
We thank Joan Bruna and Loucas Pillaud-Vivien for valuable discussions. SG acknowledges co-funding from Next Generation EU, in the context of the
National Recovery and Resilience Plan, Investment PE1 – Project FAIR ``Future
Artificial Intelligence Research''.

%% file: content/appendix.tex



\section{Dynamical analysis}%
\label{app:dynamics}

 \subsection{Expansions with Hermite polynomials \label{app:hermite}}
 
In this section we will present how expansion formulas like \eqref{eq:popu_loss_expansion} can be obtained. The idea is that, thanks to the fact that the null hypothesis distribution $\Q_0$ is a standard Gaussian, and that the planted distribution $\Q_{\text{{plant}}}$ is still quite close to a standard Gaussian, with signal on just 1 or 2 directions, it is possible to use \emph{Hermite polynomials } to expand in polynomial series all the functions of interest.

This use of \emph{Hermite polynomials} is, at this point, well known, and we refer for instance to \citep{szekely2023learning}, \citep{kunisky2019notes}, \citep{dandi2023learning} for more details on this kind of application of Hermite polynomials.
We will just recall the properties that we need:
\begin{itemize}
    \item $(h_n)_{n\in \N}$ is a family of polynomials in which $h_k$ has degree $k$.
    \item they form an \emph{orthonormal} basis for $\mathscr L^2(\R,\mathcal N(0,1)) $, which is  the Hilbert space of square integrable function with the product:
    \[\langle h_i,h_j\rangle=\E_{z\sim \mathcal N(0,1)}\left[h_i(z)h_j(z)\right]=\delta_{ij}\]
    \item this can be generalised to higher dimensions. $(H_{\amult})_{\amult \in \N^d}$ such that:
    \[
    H_\amult(x_1,\dots,x_d)=\prod_{i=1}^d h_{\amult_i}(x_i)
    \]
    form an orthonormal basis for the space $\mathscr L^2\left(\R^d,\mathcal N \left(0,\mathbbm 1_{d\times d}\right)\right)$
\end{itemize}
Moreover, it will be useful the following rewriting in our notation of lemma 1 in \citep{dandi2023learning}:
\begin{lemma} \label{lemma:hermite_change_var} Suppose $g\in \mathscr L^2(\R,\mathcal N(0,1))$ and $w\in \R$, with $||w||=1$, then $f:\R^d\to \R$ defined as $f(x):=g(w\cdot x)$ belongs to $\mathscr L^2(\R^d,\mathcal N(0,\mathbbm 1 _d))$ and has Hermite coefficients:
\begin{equation} \label{eq:coef_f1}
    C^{f}_\amult:=\underset{z\sim \mathcal N(0,\mathbbm 1_d)}\E\left[H_\amult(z)f(z)\right]=\underset{z\sim \mathcal N(0,1)}\E\left[h_{|\amult|}(z)g(z)\right]\prod_{i=1}^d w_i^{\amult_i},\quad \amult \in \N^d
\end{equation}
Suppose now 
$g\in \mathscr L^2(\R^2,\mathcal N(0,\mathbbm 1_2))$ and $u,v\in \R$, with $||u||=||v||=1$ and $u\cdot v$, then $f:\R^d\to \R$ defined as $f(x):=g(u\cdot x,v\cdot x)$ belongs to $\mathscr L^2(\R^d,\mathcal N(0,\mathbbm 1 _d))$ and has Hermite coefficients:
\begin{equation} \label{eq:coef_f2}
    C^{f}_\amult=\sum_{i+j=|\amult|}\underset{z\sim \mathcal N(0,\mathbbm 1_2)}\E\left[H_{(i,j)}(z)g(z)\right]\prod_{m=1}^d \left(\underset{\sum_m i_m=i, \sum_m j_m=j}{\sum_{ i_m+j_m=\amult_m}}u_m^{i_m}v_m^{j_m}\right) 
\end{equation}
\end{lemma}

First we will present the expansion procedure informally, and then in \cref{lemma:populoss_formula} we will see the exact statement that will be used in the subsequent proofs.

Assume the learning algorithm is a perceptron $f(w,x)=\sigma(w\cdot x)$, it can be expanded in Hermite basis with appropriate coefficients:
\[
f(w,x)=\sum_{k=0}^\infty c^{\sigma}_k h_k(w\cdot x)
\]
where
\begin{equation} \label{eq:coeff_sigma}
 c^{\sigma}_{k}=\underset{z\sim \mathcal N(0,1)}\E\left[\sigma( z)h_k(z)\right]
\end{equation}
Suppose that the data distribution is the MCM as described in \cref{sec:mixed-cumulant-model} (for simplicity we will assume that $\beta_m=0$; there are only the covariance and HOCs spikes). 
To be able to carry on the expansion it is useful to consider (as done in \citep{kunisky2019notes} and \citep{szekely2023learning}) the \emph{likelihood ratio} of $\Q_{\text{plant}}$ with respect to $\Q_0$:
\begin{equation}
\label{eq:LR} L(x):=\frac{\text{d}\Q_{\text{plant}}(x)}{\text{d}\Q_0(x)}    
\end{equation}

Note that since $\Q_{\text{plant}}$ is a perturbation of a standard Gaussian only along the directions $u,v$; then the likelihood ratio $L$ will depend only on the projections $y_u:=x\cdot u, y_v=x\cdot v$. Assuming $u,v$ to be orthogonal and the non Gaussianity $\nu$ to have a distribution that belongs to $\mathscr L^2(\R,\mathcal N(0,1))$, we can have an expansion in Hermite basis:
\[
L(x)=\sum_{i,j=0}^\infty c^{L}_{ij} h_i(y_u)h_j(y_v)
\]
where 
\begin{equation} \label{eq:coeff_L}
 c^{L}_{ij}=\underset{(z_1,z_2)\sim \mathcal N(0,\mathbbm 1_{2\times2})}\E\left[L(u\cdot z_1+v\cdot z_2)h_i(z_1)h_j(z_2)\right]
\end{equation}
So assuming to use a \emph{correlation loss}:
\[
\mathcal L(w,(x,y))=1-yf(w,x)
\]
The population loss is:
\begin{align}
\notag \mathcal L(w)&=\E_{x,y}[\mathcal L(w,(x,y)]=1+\frac12\E_{\Q_0}[\sigma(w\cdot x)]-\frac12 \E_{\Q_0}\left[L(x) \sigma(w\cdot x)\right]\\
&=\E_{x,y}[\mathcal L(w,(x,y)]=1+\frac12\underset{z\sim \mathcal N(0,1)}\E[\sigma(z)]-\frac12 \E_{\Q_0}\left[L(x) \sigma(w\cdot x)\right]\\
\notag &=1+\frac{c^\sigma_0}2-\frac12 \E_{\Q_0}\left[\left(\sum_{i,j=0}^\infty c^{L}_{ij} h_i(x\cdot u)h_j(x\cdot v)\right) \left( \sum_{k=0}^\infty c^\sigma_k h_k(w\cdot x)\right)\right] 
\end{align}
The following lemma deals with the exchanging the integral and the series.
\begin{lemma}\label{lemma:populoss_formula}
Let $\P$ be such that the likelihood ratio is square integrable $L\in \mathscr L^2\left(\R^d,\Q_d\right)$ and depends on $x$ only through its projection along 2 orthogonal directions $u, v$:
\begin{equation}
L(x)=l(u\cdot x,v\cdot x)=\sum_{i,j=0}^\infty c^{L}_{ij} h_i(x\cdot u)h_j(x\cdot v)
\end{equation}
Suppose $\sigma\in \mathscr L^2\left(\R,\Q\right)\cap\mathscr C^1(\R)$, with derivative $\sigma'\in \mathscr L^2\left(\R,\Q\right) \cap\mathscr C^0(\R)$, and expansions
\begin{align}
\sigma(y)&=\sum_{k=0}^\infty c^{\sigma}_k h_k(y)
\end{align}
then, defining $\alpha_{u}:=u\cdot w$ and $\alpha_v:=v\cdot w$, the following identity holds
\begin{gather}  \label{eq:swich_ints_proved}
\E_{\Q_0}\left[L(x)\sigma(w\cdot x)\right]=\sum_{i,j}^\infty c^{L}_{ij}c_{i+j}^{\sigma}\alpha_u^i \alpha_v^j
\end{gather}
Moreover, if $\sum_k kc^{\sigma}_k<\infty$, the population loss  $\mathcal L(w) \in \mathscr C^1( \mathcal C,\R)$, where $\mathcal C:=\left\{w\in \mathbb S^{d-1} | \alpha_u<\frac12, \alpha_v<\frac12 \right\}$ and it is also possible to switch expectations and derivatives, to get the expression:
\begin{multline}\label{eq:deriv_formula}
  \nabla \mathcal L(w)=\nabla \E_\Q\left[L(x) \mathcal L(w,(x,y))\right]=\E_\Q\left[L(x)\nabla \mathcal L(w,(x,y))\right]=
   \\ =-\frac12 \left[\left(\sum_{i=1}^\infty ic_{i0}^Lc_i^\sigma \alpha_u^{i-1}\right)u + \left(\sum_{j=1}^\infty jc_{0j}^L c_j^\sigma \alpha_v^{j-1}\right)v+\left(\sum_{i,j=1}^\infty c_{ij}^Lc_{i+j}^\sigma \alpha_u^{i-1}\alpha_v^{j-1}\left(i\alpha_v u+j\alpha_uv\right)\right)\right]
\end{multline}
\end{lemma}
\begin{proof}

First focus on \eqref{eq:swich_ints_proved}. We know that both $L$ and $\sigma(w\cdot \_)$ belong to $\mathscr L^{2}\left(\R^{d},\Q_0\right)$, hence we can see the integral in \eqref{eq:swich_ints_proved}  as the $\mathscr L^2$ inner product, and compute it using the series expansion and \cref{lemma:hermite_change_var}.
\begin{align*}
    \E_{\Q_0}\left[L(x)\sigma(w\cdot x)\right]&=\sum_{\amult \in \N^d} c^L_\amult c^{\sigma(w\cdot)}_\amult=\sum_{k=0}^\infty\sum_{|\amult|=k} c^L_\amult c^\sigma_k\prod_{l=1}^d w_l^{\amult_l}\\
    &=\sum_k\sum_{i+j=k}c_{ij}^Lc^\sigma_k\sum_{|\amult|=k}\prod_{m=1}^d w_m^{\amult_m}\left(\underset{\sum_m i_m=i, \sum_m j_m=j}{\sum_{ i_m+j_m=\amult_m}}u_m^{i_m}v_m^{j_m}\right)\\
    &= \sum_k\sum_{i+j=k}c_{ij}^Lc^\sigma_k\prod_{m=1}^d \left(\underset{\sum_m j_m=j}{\sum_{\sum_m i_m=i}}(u_mw_m)^{i_m}(v_mw_m)^{j_m}\right)\\
    & =\sum_k\sum_{i+j=k}c_{ij}^Lc^\sigma_k \alpha_u^i\alpha_v^j
\end{align*}
Moreover it can be verified directly that, under the assumption that
$\sum_k kc^{\sigma}_k<\infty$, and $w\in \mathcal C$ the series in  \eqref{eq:deriv_formula} converges uniformly, hence also the second part of the statement holds.
\end{proof}

\subsection{Properties of the MCM model}
In this section we will apply the Hermite expansion to the Mixed Cumulant Model and note the key properties that will allow to prove the Propositions of \cref{sec:odes}.
\subsubsection{Hermite coefficients \label{app:formulas}}
Starting from the single direction models we have that:
\begin{itemize}
    \item If $\beta_m=\beta_v=0$ only the covariance spike survives and the model is Gaussian, hence the population loss is:
    \begin{equation} \label{eq:sp_wish_popu_loss}
    \mathcal L(w)=\ell(\alpha_u)=1-\frac{\beta_uc_{2}^\sigma}{4}\alpha_{u}^2
    \end{equation}
    \item if $\beta_m=\beta_u=0$ only the cumulants spike survives. To compute the Hermite coefficients we can rely on lemma 13 form \cite{szekely2023learning} to get:
     \begin{equation} \label{eq:sp_cum_popu_loss}
    \mathcal L(w)=\ell(\alpha_v)=1-\frac12\sum_{j\ge 3}\frac{c_j^\sigma}{j!}\left(\frac{\beta_v}{1+\beta_v}\right)^{j/2}\mathbb E[ h_j(\nu)]\alpha_{v}^j
     \end{equation}
     So if $\nu=$Radem$(1/2)$ we have that the leading term is
     \[
     \mathcal L(w)=\ell(\alpha_v)=1+\frac{c_4^\sigma}{4!}\left(\frac{\beta_v}{1+\beta_v}\right)^{2}\alpha_{v}^4+o(\alpha_v^4)
     \]
     Note the change of sign, due to $\mathbb E[ h_4(\nu)]=-2$. This is why we need to ask $c_4^\sigma<0$ to have that the population loss is decreasing for small $\alpha_v$.
     \item consider now the case with both spike and  \emph{independent} latent variables $\lambda, \nu$. It is easy to see that, thanks to the independence, the coefficients split  $c^L_{ij}=c^{cov}_ic^{cumulant}_j$, so the leading terms are:
     \begin{equation}
           \label{eq:mixed_cumulants_popu_loss_approx}
    \mathcal L(w)=\ell(\alpha_u,\alpha_v)=1-\frac12\underbrace{\frac{\beta_uc_{2}^\sigma}{4}}_{c_{20}}\alpha_{u}^2-\underbrace{\frac{-c_4^\sigma}{4!}\left(\frac{\beta_v}{1+\beta_v}\right)^{2}}_{c_{04}}\alpha_{v}^4+o(\alpha_v^4)
    \end{equation}
    \item in case of latent variables with positive correlation, we need to add mixed terms, the leading one is $c_{11}=c_2^\sigma\sqrt{\frac{\beta_u\beta_v}{1+\beta_v}}\E[\lambda \nu]$. So the first few terms of the expansion of the population loss are:
    \begin{equation}
        \mathcal L(w)=\ell(\alpha_u,\alpha_v)=1-\frac12\left(c_{20}\alpha_u^2+c_{11}\alpha_{u}\alpha_v+c_{04}\alpha_v^4\right)+o(\alpha_u\alpha_v)
    \end{equation}
    note that the presence of $c_{11}>0$ makes the term $c_{04}\alpha_v^4$ to become a sub-leading contribution, lowering the information exponent on the $v$ direction.
\end{itemize}
\subsubsection{Assumption on the sample wise error\label{assumptionB}}
An important quantity for proving the Propositions in \cref{sec:odes} is the \emph{sample-wise error}
\[
H^{\mu}_d(w):= \mathcal L_d\left(w,(x^\mu,y^\mu)\right)-\mathcal L_d(w)
\]
we will ask the same assumptions of \citep{benarous2021online}:
\begin{align*}
   \sup_{w\in \mathbb S^{d-1}} \E\left[\left(\nabla_{sph}H_d(w)\cdot u\right)^2\right]&\le C_1\\
    \sup_{w\in \mathbb S^{d-1}} \E\left[\left(\nabla_{sph}H_d(w)\cdot v\right)^2\right]&\le C_1\\
    \sup_{w\in \mathbb S^{d-1}} \E\left[||\nabla_{sph}H_d(w)||^{4+\iota}\right]&\le C_2d^{\nicefrac{4+\iota}{2}} \qquad \text{for some }\iota>0
\end{align*}
simple calculations ensure that it is sufficient to ask that \begin{itemize}
    \item $\Q_{\text{plant}}$  has finite moments up to the  $8$-th order,which is satisfied when taking $\nu\sim$Rademacher$(1/2)$.
    \item the following holds:
    \begin{align*}
         \sup_{w\in \mathbb S^{d-1}} \E\left[\sigma'(w\cdot x)^{4}\right]&\le C\\
            \sup_{w\in \mathbb S^{d-1}} \E\left[\sigma'(w\cdot x)^{8+2\iota}\right]&\le C\\
    \end{align*}
    which we have assumed in \eqref{eq:sigma hp}
\end{itemize}

\subsection{Proof for the SGD analysis}
Here we will provide proofs for the Propositions presented in \cref{sec:odes}. All of them rely heavily on the thorough analysis of spherical perceptron dynamics proved in \citep{benarous2021online}. 

Note that \cref{prop:spiked_wish} is a well known result, very close to the examples provided in section 2 of \cite{benarous2021online}, and  we can verify that a straightforward application of  
theorems 1.3 and 1.4  from \cite{benarous2021online}, in the case $k=2$, proves it.

A different situation happens for \cref{prop:spike_cumulant}, the spiked cumulant model is quite recent so its SGD dynamics have still to be investigated. However it is very quick to verify that  also in this case it is possible to apply theorems 1.3 and 1.4  from \cite{benarous2021online}. Assumption B in \cite{benarous2021online} is met thanks to \cref{assumptionB}. The monotonicity assumption could be harder to check, but since we are interested only in the search phase, it is sufficient to satisfy that $\mathcal L(w)=\ell(\alpha_v)$ is monotone only in $(0,\rho)$ for any $\rho>0$ (assumption $A_\rho$ introduced in section 3.1 of \cite{benarous2021online}), which is clear from \eqref{eq:sp_cum_popu_loss}. 
Hence we can apply the case $k=4$ of theorems 1.3 and 1.4  from \cite{benarous2021online}, to conclude the proof.
\subsubsection{Proof of proposition \ref{prop:negative_result}}

    This proposition can be considered a Corollary of \emph{Theorem 1.4} from \citep{benarous2021online}, the only differences is that the data distribution has an additional direction $u$ and that we are considering also lower learning rates $\delta=o(1)$ instead of $\delta=o(\frac 1d)$ (note that at page 10 in \cite{benarous2021online} this extension to larger learning rates is already mentioned).
    
    On one hand we will verify that thanks to the fact that the auxiliary variable $\lambda$, $\nu$ are independent, then the additional direction $u$ makes no impact in the dynamic of the overlap $\alpha_v$. On the other hand, the condition $n\le \frac{d}{\delta_d^2}$ ensures that we are always inside the time horizon considered in \citep{benarous2021online}. We will discuss how to deal with these two changes in the next two paragrphs.

    \paragraph{Direction $u$ can be neglected}

Since $\nabla_{sph}\mathcal L(w_t,(x^t,y^t))$ is orthogonal to $w_t$ we have that $||\tilde\alpha_{v,t}||=||v\cdot \tilde w_t||\ge 1 $. Hence we have that it $\alpha_{v,t}>0$, then
\begin{align*}
    \alpha_{v,t}&\le \tilde \alpha_{v,t}=\alpha_{v,t-1}-\frac{\delta}d \nabla_{sph}\mathcal L(w_{t-1},(x^t,y^t))\cdot v\\
    &=\alpha_{v,t-1}-\frac{\delta}d \nabla_{sph}\mathcal L(w_{t-1})\cdot v-\frac{\delta}d \nabla_{sph}H^t(w_{t-1})\cdot v
\end{align*}
Since the sample-wise error satisfies the same properties required in \citep{benarous2021online} (as checked in \cref{assumptionB}) we just need to ensure that the drift term also can be handled in the same way as in \citep{benarous2021online}. There, the monotonicity assumption on the population loss allowed to estimate $-v\cdot\nabla_{sph}\mathcal L(w)\le C \alpha_v^3$ for $\alpha_v>0$. So the aim will be to get a similar bound. Note that we can assume $\alpha_v$ smaller than $\eta$, hence we can apply \eqref{eq:mixed_cumulants_popu_loss_approx} and \cref{lemma:populoss_formula} to get:
\begin{align*}
    -\nabla_{sph}\mathcal L(w)\cdot v&=(\left(v-\alpha_vw\right)^\top \nabla L(w)\\& =\partial_v \ell-\alpha_v\left(\alpha_u\partial_u \ell+\alpha_v\partial_v \ell\right)\\
    &=  \ 4c_{04}\alpha_v^{3}-c_{20}\alpha_{u}^2\alpha_v+o(\eta^3)\\
     &\le  5c_{04}\alpha_v^{3}
    \end{align*}   
We can now use this inequality to reach the same estimates as in p 37 of \cite{benarous2021online}, and carry on with their proof, in the case $\delta=o(\frac 1d)$. We only need to show that we can take larger step size, at the cost of taking $n\ll \frac{d}{\delta_d^2}$.
\paragraph{Larger step size}
We note that in the proof of  \emph{Theorem 1.4} from \citep{benarous2021online} the main use of the assumptions $\delta=o(\frac 1d)$ was to ensure that $\frac nd\delta_d^2=o(1)$ which is required to use the following form of Doob's maximal inequality for martingales:
\begin{equation}
    \label{eq:mart_estimfirst} \sup_{w_o \in \mathbb S^{d-1}} \P_{w_0}\left(\max_{t\le n} \frac\delta d\left|\sum_{j=0}^{t-1} \nabla_{sph} H^{j+1}(w_j)\cdot u \right|\ge r\right)\le \frac{2n\delta^2C}{d^2r^2},
\end{equation}
where $\P_{w_0}$ means that we are conditioning that the initial weight is $w_0$, $C$ comes from \cref{assumptionB} and $r>0$ is a free parameter. To carry out the proof, the noise level needs to be of order $\frac{1}{\sqrt d}$, hence we want to be able to take $r\approx d^{-1/2}$ and have that  the probability in \cref{eq:mart_estimfirst} goes to 0, which is guaranteed by the condition $\frac nd\delta_d^2=o(1)$.
Hence we can follow \cite{benarous2021online} p 38 and get that the trajectory will be below $\eta$ for all $t\le \tilde t:=C_\eta \frac{d^2}{\delta}$, which is larger than our horizon, completing the proof.
    

 \subsubsection{Proof of proposition \ref{prop:positive_result}}
 We will focus on the case of correlated latent variables and prove weak recovery of both spikes. This will be a proof also of the weak recovery of the covariance spike in the case of independent latent variable (which is actually much simpler and is essentially a straight-forward verification of \cite{benarous2021online} theorem 1.3 to our slightly different setting).

 First note that we can restrict to initialisation with positive overlap, $\alpha_{u,0},\alpha_{v,0}>0$: if for example the initialisation was $\alpha_{u,0}<0$ then we can just substitute $\hat u:= -u $, $\hat \lambda=-\lambda$ so that the sign changes cancel each other and $\hat x=x$ but $\hat \alpha_{u,0}=-\alpha_{u,0}$. In the case of uncorrelated latents, the distribution would not change: $\hat x\sim x$, due to the symmetry of the latent variables and independents of the two directions, hence the new problem is equivalent to the previous one. 

  In case of positive correlation of the latent variables,  we need to use our assumption that there is no sign mismatch in the initialisation. Unlike the uncorrelated case, we cannot change the sign of one of the axes without changing also the other, since it would flip the sign of the correlation.
 So there are only two cases: either $\alpha_{u,0},\alpha_{v,0}>0$, which is the desired initialisation; or  $\alpha_{u,0},\alpha_{v,0}<0$, where we can flip the signs along \emph{both the spiked directions}:
 \[\hat u=-u, \ \hat \lambda=-\lambda, \quad \hat v=-v, \ \hat \mu=-\mu
 \] So that $\hat x=x$ and $(\hat \lambda,\hat \nu)\sim (\lambda,\nu)$ keeping a positive correlation and $\hat \alpha_{u/v,0}=-\alpha_{u/v,0}$.
  
 So from now on we can assume $\alpha_{u,0},\alpha_{v,0}>0$.
 Hence, for $d$ large enough, $w_0\in K_\eta:=\left\{w: (\alpha_u,\alpha_v)\in [0,\eta]^2\right\}$,  for $\eta>0$ small, so we are able to use the expansions from \cref{app:formulas}.
 
 Recall from the statement that we assume $n=\theta_d d$, with $\theta_d\gtrsim \log^2 d$, and choose the step size $\delta=\delta_d$ such that $\nicefrac{1}{\theta_d}\ll \delta\ll \nicefrac{1}{\sqrt{\theta_d}}$.
 To prove weak recovery, we will show that 
there are $\eta_u,\eta_v\le \eta$ such that, defining the $\tau_u,\tau_v$ as the first times in which $\alpha_{u,t}\ge \eta_u$ and $\alpha_{v,t}\ge \eta_v$, then the probability that $\tau_u,\tau_v\le n$ goes to 1 as $d\to \infty$.  
 
\paragraph{Weak recovery of the covariance spike} 
Given $\gamma>0$, let $E_\gamma:=\left\{w \ \big |\ \alpha_{u}\ge \frac{\gamma}{\sqrt d},\alpha_v\ge \frac{\gamma}{\sqrt d}\right\}$.
 Since we consider the limit $d\to \infty$ and $w_0\sim$Unif$(\mathbb S^{d-1}$), we can apply Poincaré lemma and have that for any $\varepsilon>0$ we can find $\gamma>0$ such that for all $d$ sufficiently large
 \begin{equation}
     \label{eq:Eprob} \P\left(w_0\in E_\gamma\right)\ge 1-\varepsilon
 \end{equation}
 So, fix $\varepsilon>0$ and from now on we will assume that $w\in E_\gamma$ for some $\gamma>0$ sufficiently small (or $d$ sufficiently large) so that \eqref{eq:Eprob} holds.
 
The first step is to prove that the \emph{difference inequality} from  \emph{Proposition 4.1} from \cite{benarous2021online} holds also in our case for $(\alpha_{u,t})_{t>0}$.   
So starting from the update equation  \eqref{eq:onlineSGD} and taking the scalar product by $u$ we get
\begin{align*}
    \tilde\alpha_{u,t+1}&=\alpha_{u,t}-\frac \delta d u\cdot\nabla_{sph} \mathcal L\left(w,(x_t,y_t)\right)\\
  &=  \alpha_{u,t}-\frac \delta d u\cdot \left(\nabla_{sph} \mathcal L\left(w\right)+\nabla_{sph} H^t(w)\right)
\end{align*}
Where we used the definition of \emph{sample-wise loss} $H^t(w)=\mathcal L\left(w,(x_t,y_t)\right)- \mathcal L\left(w\right)$. Now we need to normalise $\alpha_{u,t}=\nicefrac{\tilde \alpha_{u,t}}{||\tilde \alpha_{u,t}||}$ but instead we just estimate the denominator, using the fact that if $\alpha_u,\alpha_v\le \frac12$ by \cref{lemma:populoss_formula} $\mathcal L\in\mathscr C^1(\mathcal C,\R)$. So restricting in $\mathcal K_\eta$  for any $\eta<\frac 12$ we have that there exists $A>0$:
 \begin{equation} \label{eq:Adef}
 \sup_{w\in \mathcal K_\eta} |\nabla_{sph} \mathcal L\left(w\right)|\le A
 \end{equation}
 and letting
 \begin{equation} \label{eq:Thetadef}
 \Theta_t:=\left|\frac{1}{\sqrt d}\nabla_{sph} H^t\left(w_{t-1}\right)\right|^2
 \end{equation}
 It is straightforward to estimate:
 \[
 1\le |\tilde \alpha_{u,t}|\le 1+\underbrace{\delta^2\left(\frac{A}{d^2}+\frac{\Theta_t}d\right)}_{\zeta}
 \]
 Then we can use the chain of inequalities: $\frac{1}{|\tilde \alpha_{u,t}|}\ge \frac1{1+\zeta}\ge 1-\zeta$ to get:
 \begin{multline}\label{eq:step_estimate}
     \alpha_{u,t}\ge \alpha_{u,t-1}-\frac{\delta}d\nabla_{sph}\mathcal L(w)\cdot u -\frac\delta du\cdot\nabla_{sph}H^t(w_{t-1})-\delta^2\left(\frac A{d^2}+\frac{\Theta_t}d\right)|\alpha_{u,t-1}|\\-\delta^3\left(\frac A{d^2}+\frac{\Theta_t}d\right)\left(\frac{1}d|\nabla_{sph}\mathcal L(w)\cdot u|+\frac1 d|u\cdot\nabla_{sph}H^t(w_{t-1})|\right)\end{multline}

Now, one of the key ideas employed in \citep{benarous2021online} is to split the random term
\begin{equation}
    \delta^2\frac{\Theta_t}d|\alpha_{u,t-1}|= \delta^2\frac{\Theta_t\mathbbm 1_{\Theta_t\le \hat \Theta}}d|\alpha_{u,t-1}|+\delta^2\frac{\Theta_t\mathbbm 1_{\Theta_t\ge \hat \Theta}}d|\alpha_{u,t-1}|
\end{equation}
Where $\hat \Theta$ is a real parameter to be tuned carefully.
So we can regroup the terms in \eqref{eq:step_estimate} in two families (for the sake of brevity, we use the notation $\Lambda :=|\nabla_{sph}\mathcal L(w)\cdot u|+|u\cdot\nabla_{sph}H^t(w_{t-1})|$):
 \begin{multline}\label{eq:step_estimate_families}
     \alpha_{u,t}\ge \alpha_{u,t-1}\overbrace{-\frac{\delta}d\nabla_{sph}\mathcal L(w_{t-1})\cdot u- \delta^2\frac{\Theta_t\mathbbm 1_{\Theta_t\le \hat \Theta}}d|\alpha_{u,t-1}|}^{F_1}+\\ \underbrace{-\frac\delta du\cdot\nabla_{sph}H^t(w_{t-1}) -\delta^2\left(\frac A{d^2}+\frac{\Theta_t\mathbbm 1_{\Theta_t\ge \hat \Theta}}d\right)|\alpha_{u,t-1}|-\frac{\Lambda\delta^3}{d}\left(\frac A{d^2}+\frac{\Theta_t}d\right)}_{F_2}\end{multline}

The family $F_2$ is made up by terms that are either always negligible, or that can cause problems only if \emph{very unlikely realisations} of $H$ and $\Theta$ happen.
They are estimated in lemmas 4.3 and 4.5 of \cite{benarous2021online}, and, since our assumption \ref{assumptionB} implies assumption B in \cite{benarous2021online}, we are able to replicate those estimates also in our case.

On the contrary, to handle $F_1$ we need to make an adjustment: \cite{benarous2021online} relied on their formula (4.12), that was derived from the monotonicity and mono-dimensionality of the population loss, which we do not have. 
However, consider the gradient of the population loss in direction $u$. Applying \cref{lemma:populoss_formula} and substituting the values from \cref{app:formulas} we have that $ -u\nabla_{sph} \mathcal L(w)=\frac12(c_{20}\alpha_u+c_{11}\alpha_v)+o(\eta)$, since both $c_{20}=2\beta_uc^\sigma_2>0$ and $c_{11}>0$, taking $\eta$ small enough and assuming $(\alpha_u,\alpha_v)\in \mathcal K_\eta$ we have that:
\begin{equation}\label{eq:monotonicity_estim_u}
    \frac14\left(c_{20}\alpha_u+c_{11}\alpha_v\right)\le  -u\nabla_{sph} \mathcal L(w)\le c_{20}\alpha_u+c_{11}\alpha_v
\end{equation}
Hence, if we assume $\alpha_v\ge0$, we get $ -u\nabla_{sph} \mathcal L(w)\ge \frac14 c_{20}\alpha_u$, which is exactly what we needed to replicate equation (4.12) in \cite{benarous2021online}, so as long as $\alpha_{u/v,t-1}>0$ we have that:
\begin{equation}\label{eq:lemma4.2applied}
    \alpha_{u,t}\ge \alpha_{u,t-1}-\frac{\delta}{d}\left(\frac{c_{20}}4- \delta\Theta_t\mathbbm 1_{\Theta_t\le \hat \Theta}\right)\alpha_{u,t-1}+F_2
\end{equation}
Hence we are back on the setting studied in \cite{benarous2021online}, and using that $\delta \lesssim \frac 1{\log d}$ we can apply also lemma 4.2 and proposition 4.4 in \cite{benarous2021online}.

So, conditioning on positivity of $(\alpha_{v,t})_{t\le\tau}$ we can apply to $(\alpha_{u,t})_t$ \emph{Proposition 4.1} from \cite{benarous2021online}, that implies
\begin{equation}
    \label{eq:prop4.1_to_u}
 \begin{aligned}
 \lim_{d\to \infty} &\inf_{w_0\in E_{\nicefrac{\gamma}{\sqrt d}}} \P_{w_0}\left(\alpha_{u,t}\ge \frac{\alpha_{u,0}}2+\frac{\delta}{8d}\sum_{j=0}^{t-1}c_{20}\alpha_{u,j} \forall \ t\le \tau \ \Bigg | \alpha_{v,t}\ge0 \ \forall t\le \tau\right)=1\\
    &\text{where $\tau$ is a stopping time defined as }\qquad \tau:=\inf \left\{t\le n\big |\  \alpha_{u,t}\le \frac{\gamma}{2\sqrt d}\text{ or }  \alpha_{u,t}\ge{\eta}\right\}
    \end{aligned}
\end{equation}
where $\P_{w_0}$ denotes the probability conditioned on starting from $w_0$.

 Recall discrete Gronwall inequality 
 \[
     x_n\ge a+b\sum_{i=0}^{n-1}x_i \quad \Longrightarrow \quad x_n\ge a\left(1+b\right)^n, \qquad a,b>0
 \]
 Using it, we get that, if the event in \eqref{eq:prop4.1_to_u} happens, then
 \begin{equation} \label{eq:lower_bound_u}
 \alpha_{u,t}\ge \frac{\alpha_{u,0}}2\left(1+\frac{\delta c_{20}}{8d}\right)^t \qquad \forall t \le \tau
 \end{equation}
 So, recalling that we are conditioning to be in $E_\gamma$, so $\alpha_{u,0}\ge \frac{\gamma}{\sqrt d} $, \eqref{eq:lower_bound_u} proves that $\alpha_{u,t}$ will have surpassed $\eta_u$ after a number of steps $t_u^*$:
 \begin{equation}\label{eq:tstaru}t_u^{*}=\frac{\log(2\eta_u\sqrt d)-\log  \gamma}{\log(1+\frac{\delta c_{20}}{8d})} \approx C\frac{d}{\delta}\log d
\end{equation}
This proves weak recovery of the covariance spike in the time require ending the first part of \cref{prop:positive_result}.
 \paragraph{Weak recovery of the cumulant spike}
 Now turn to $(\alpha_{v,t})_t$. Starting from the equivalent of \eqref{eq:step_estimate_families}, we cannot apply the same reasoning this time because 
 \begin{equation} \label{eq:monoton_estim_v}-v\cdot\nabla_{sph}\mathcal L(w)=c_{11} \alpha_u(1+o(\eta))+4c_{04}\alpha_v^3(1+o(\eta)) \qquad \alpha_u,\alpha_v\le \eta.
 \end{equation}
 There is no linear term in $\alpha_v$, hence we cannot regroup the term $\delta\Theta_t\mathbbm 1_{\Theta_t\le \hat \Theta}|\alpha_{v,t-1}|$ as did in \eqref{eq:lemma4.2applied}. 

 We will instead prove a the following:
 \begin{equation}\label{eq:dreambound}
 \alpha_{v,t}\ge \frac{\tilde\gamma}{2\sqrt d}\left(1+\frac{\delta \min(c_{20},c_{11})}{8d}\right)^t
 \end{equation}
 for some $\tilde \gamma$ small enough.
 To show it,  we apply lemma 4.3, 4.5 from \cite{benarous2021online}, and \eqref{eq:monoton_estim_v}, to get that:
 \begin{align*}
 \lim_{d\to \infty} &\inf_{w_0\in E_{\nicefrac{\gamma}{\sqrt d}}} \P_{w_0}\left(\alpha_{v,t}\ge \frac{\alpha_{v,0}}2+\frac{\delta}{4d}\sum_{j=0}^{t-1}c_{11}\alpha_{u,j}+4c_{04}\alpha^3_{v,j}-\frac{\delta^2}{d}\mathbbm1_{\Theta_j\le \hat \Theta}\alpha_{v,j} \quad \forall \ t\le \tilde \tau \right)=1\\
    &\text{where $\tilde \tau$ is a stopping time defined as }\qquad \tilde \tau:=\inf \left\{t\le n\big |\  \alpha_{v,t}\le \frac{ \gamma}{2\sqrt d}\text{ or }  \alpha_{v,t}\ge{\eta}\right\}
    \end{align*}
    We need to divide in 2 cases to be able to carry one the estimates:
 suppose first that $\alpha_{u,s}\ge \alpha_{v,s}$ for all $s\le t$. Hence $\delta^2\alpha_{v,t}\le \delta^{2}\alpha_{u,t}$ hence we can also apply proposition 4.4 to get:
 \begin{equation}\label{eq:diff_ineq_v}
  \begin{aligned}
 \lim_{d\to \infty} &\inf_{w_0\in E_{\nicefrac{\gamma}{\sqrt d}}} \P_{w_0}\left(\alpha_{v,t}\ge \frac{\alpha_{v,0}}2+\frac{\delta}{8d}\sum_{j=0}^{t-1}c_{11}\alpha_{u,j} \quad \forall \ t\le \mathcal S_u \right)=1\\
    &\text{where $\mathcal S_u$ is defined as }\, \mathcal S_u:=\inf \left\{t\le n\big |\  \alpha_{v,t}\le \frac{\tilde \gamma}{2\sqrt d}\text{ or }  \alpha_{v,t}\ge{\eta}\text{ or } \alpha_{v,t}\ge \alpha_{u,t}\right\}
    \end{aligned}
     \end{equation}
 This implies that as long as $\alpha_v$ is below $\alpha_u$, it satisfies the same kind of Gronwall inequality, with different coefficients, leading to:
 \[
 \alpha_{v,t}\ge \frac{\alpha_{v,0}}2\left(1+\frac{\delta c_{11}}{8d}\right)^t \qquad t\le \mathcal S_u
 \]
 so it satisfies \eqref{eq:dreambound} for this selection of times. Now note that we actually proved a stronger property. If at any time $t$: $\frac{\tilde \gamma}{2\sqrt d}\left(1+\frac{\delta \min(c_{11},c_{20})}{8d}\right)^t\le \alpha_{v,t}\le \alpha_{u,t}$, then we know that same reasoning applies and, calling $\mathcal S^{(2)}_u $the next instant such that with $\alpha_{v,\mathcal S^{(2)}_u }\ge\alpha_{u,\mathcal S^{(2)}_u }$ we get that for all times $t\le s\le \mathcal S^{(2)}_u $:
\[ \alpha_{v,s}\ge \frac{\alpha_{v,t}}{2}\left(1+\frac{\delta c_{11}}{8d}\right)^{s-t}
\]
hence  $\alpha_{v,s}$ satisfies \eqref{eq:dreambound}.

Consider now the other case: if at any instant $t$ we have that $\alpha_{v,t-1}\ge\alpha_{u,t-1} $  we want to prove that it cannot happen that $\alpha_{v,t}\le \frac{\tilde\gamma}{2\sqrt d}\left(1+\frac{\delta \min(c_{20},c_{11})}{8d}\right)^t $. To do this we use an estimate on the update equation (it can be easily verified that comes from the application of lemma 4.3,4.5 to formula (4.11) of \cite{benarous2021online}, with their choice $\hat \Theta=d^{\frac 12-\frac14\iota}$):
\begin{align*}
\alpha_{v,t}&\ge \left(1-\frac{\delta^2}{d^{1/2+1/4\iota}}\right)\alpha_{v,t-1}+\frac{c_{11}\delta}{8d}\alpha_{u,t-1}-\frac{\tilde \gamma}{5\sqrt d}\\&\ge
\left(1-\frac{\delta^2}{d^{1/2+1/4\iota}}+\frac{c_{11}\delta}{8d}\right)\alpha_{u,t-1}-\frac{\tilde \gamma}{5\sqrt d}\\
&\ge
\left(1-O\left(d^{-1/2}\right)\right)\frac{\alpha_{u,0}}2\left(1+\frac{\delta c_{20}}{8d}\right)^t-\frac{\tilde \gamma}{5\sqrt d}
\end{align*}
Which verifies our requirement by taking $\tilde \gamma$ small enough with respect to $\gamma$ and $d$ large enough.
 
Hence we verified that conditioning on the the events \eqref{eq:Eprob} and \eqref{eq:prop4.1_to_u} (that have arbitrarily large probabilities), also inequality \eqref{eq:dreambound} holds. This means that we can drop all the positivity requirements and have that, as long as $0<\alpha_{u,t},\alpha_{v,t}\le \eta$, then with high probability $(\alpha_{u,t})_t$ satisfies \eqref{eq:lower_bound_u} until it reaches $\eta_u$ and $(\alpha_{v,t})_{t}$ satisfies \eqref{eq:dreambound} until it reaches $\eta_v$, hence we need to take $n$ larger than the maximum of those bounds, which is \eqref{eq:dreambound}:

\begin{equation}\label{eq:tstar}n>t_v^{*}=\frac{\log(2\eta\sqrt d)-\log \tilde \gamma}{\log(1+\frac{\delta \min(c_{20},c_{11})}{8d})} \approx C\frac{d}{\delta\min(c_{20},c_{11})}\log d
\end{equation}
Where the first constant $C$ is just a number, does not depend on $\eta,\tilde \gamma$, for $d$ large enough.
Substituting the constraint $\delta \ll \frac{1}{\log d}$ (that we used in order to apply proposition 4.1 from \cite{benarous2021online}) we get that $n\gtrsim d\log^2 d$ is a sufficient sample complexity for weak recovery if $\delta$ is as close as possible to the upper bound $\frac{1}{\log d}$. If instead we take a smaller learning rate $\delta \approx \frac1{d^2}$ we have that if $n\gg d^2\log d$ weak recovery is guaranteed, as shown in \cref{tab:results}.

So the proposition is proved under the assumption that all the trajectories live in $K_\eta$ with $\eta$ sufficiently small. To conclude we need to make sure that it cannot happen that one of $(\alpha_{u/v,t})_t$ surpasses $\eta$ while the other is still close to 0 and far away from $\eta_{u/v}$. Note that up until now we did not choose explicitly values for $\eta_{u/v}$, so it will be sufficient to take them way smaller than $\eta$  to be sure that there will be no problems.

To be more precise, we will use the following bound on the maximum progress achieved in a fixed time interval.
\paragraph{Bound on maximum progress}
Supposing $\alpha_{u,t}>0$ (all that we will do here works also for $\alpha_v,t$), since the direction of the spherical gradient is tangent to the sphere it follows that $ ||\tilde w_{u,t}||>1$, hence
\begin{equation}
 \alpha_{u,t}\le \tilde\alpha_{u,t}= \alpha_{u,t-1}-\frac{\delta}d\nabla_{sph}\mathcal L(w_{t-1})\cdot u -\frac\delta du\cdot\nabla_{sph}H^t(w_{t-1})
\end{equation}
Applying this inequality iteratively in a time interval $[t_1,t_2]$, assuming $\alpha_{u,t}>0$ in all the interval, we get that
\[
\alpha_{u,t_2}\le  \alpha_{u,t_1}-\sum_{j=t_1}^{t_2-1}\frac{\delta}d\nabla_{sph}\mathcal L(w_{j})\cdot u +\frac\delta du\cdot\nabla_{sph}H^{j+1}(w_{j})
\]
To estimate the martingale term we use the estimate from lemma 4.5 in \cite{benarous2021online} (which is an application of Doob maximal inequality):
\begin{equation}
    \label{eq:mart_estim} \sup_{T\le n} \sup_{w_o \in \mathbb S^{d-1}} \P_{w_0}\left(\max_{t\le T} \frac 1{\sqrt T} \left|\sum_{j=0}^{t-1} \nabla_{sph} H^{j+1}(w_j)\cdot u \right|\ge r\right)\le \frac{2C_1}{r^2},
\end{equation} where $r>0$ is a positive parameter and $C_1$ comes from assumption \ref{assumptionB}. Together with \eqref{eq:Adef}, we have that with probability larger than $1-\frac{2C_1}{r^2}$
\begin{equation}
    \label{eq:max_disp}
    \alpha_{u,t_2}- \alpha_{u,t_1}\le \frac{\delta}{Ad}(t_2-t_1) +\frac{\delta}{d}\sqrt{t_2-t_1}r 
\end{equation}
assume that at $t_1$ we have that $\alpha_{u,t_1}\ge \eta_1$ and that at $\alpha_{u,t_2}=\eta_2$ while $\alpha_{u,t}\ge \eta_1$ for all $t \in[t_1,t_2]$, choosing $r=\sqrt{t_2-t_1}$ we get that with probability larger than $1-\frac{2C_1}{t_2-t_1}$
\begin{equation} \label{eq:t1t2_low_bound}
t_2-t_1\ge C(\eta_2-\eta_1)\frac{d}\delta\gg C(\eta_2-\eta_1) d \log d
\end{equation}
so in the limit $d\to \infty$ with probability converging to 1 $t_2-t_1\gg d\log d$.

\paragraph{Application of the bound}
Suppose that at time $t_2$ $\alpha_{u,t_2}= \eta=:\eta_u$, hence there exists $t_1$ such that $\alpha_{u,t}\ge \frac{\eta} 2$ for all $t \in [t_1,t_2]$. By \eqref{eq:t1t2_low_bound}, $t_2-t_1\ge \frac{C\eta}2 d\log d$. If $\alpha_{v,t}$ is larger than $\alpha_{u,t}$ for any $t$ in this interval, the statement is already verified by taking $\eta_{v}=\eta_u$, so assume $\alpha_{v,t}\le\alpha_{u,t}$ for $t\in [t_1,t_2]$. We can apply \cref{eq:diff_ineq_v} to get that with probability converging to 1 in the limit $d\to \infty$:
\[
\alpha_{v,t_2}\ge \frac{\alpha_{v,t_1}}2+\frac{\delta}{8d}\sum_{j=t_1}^{t_2-1}c_{11}\alpha_{u,j}\ge c_{11}\frac{\delta\eta_u}{16d}(t_2-t_1)\ge C \eta_u^2
\]
where the constant $C$ does not depend on $d$, hence we can take $\eta_v:=C\eta_u^2$ there is weak recovery for the cumulant spike $v$ even in the "unlucky" case where the overlap along $u$ reaches $\eta$ very fast.

The vice versa can be done analogously: assume  that at time $t_2$ $\alpha_{v,t_2}=\eta=: \eta_v$. Again, let  $t_1$ such that $\alpha_{v,t}\ge \frac{\eta} 2$ for all $t \in [t_1,t_2]$. By \cref{eq:t1t2_low_bound}, $t_2-t_1\ge \frac{C\eta}2 d\log d$. Now note that by \eqref{eq:monotonicity_estim_u} $-u\cdot \nabla_{\text{sph}}L(w)\ge \frac{c_{11}\alpha_{v,t}}4$, so the same reasoning as before applies leading to weak recovery with $\eta_u:=C\eta_v^2$.

\section{Details on the experiments}

\subsection{Experiments with two-layer networks}%
\label{app:figure-details}

\paragraph{Training} In all our experiments with two-layer neural networks on \emph{synthetic} data, be it on the mixed cumulant model or in the teacher-student setup, we
trained two-layer neural networks using online stochastic gradient with
mini-batch size 1. Note that we used different learning rates for the first and
second layer, with the second layer learning rate $\eta_2 = \epsilon \eta_1$ and
$\epsilon=0.01$ unless otherwise noted. It has been noted several times that
rescaling the second-layer learning rate in this way ensures convergence to a
well-defined mean-field limit~\citep{riegler1995line, berthier2023learning}

\subsubsection{Figure \ref{fig:figure1}}

\paragraph{CIFAR10} We trained a two-layer neural network with $m=32^2 * 4 = 4096$ hidden neurons on grayscale CIFAR10 images to ensure the same ratio between hidden neurons and input dimension as in our experiments with synthetic data. We set the mini-batch size to 128, weight decay to $5 \cdot 10^{-4}$, and momentum to 0.9. Images were transformed to grayscale using the pyTorch conversion function. We trained for 200 epochs on the cross-entropy loss and used a cosine learning rate scheduler. The final test accuracy of the networks was just below $50\%$.

\paragraph{Synthetic data distributions} We trained the two-layer neural network
on a data set sampled from the mixed cumulant model with signal-to-noise ratios
$\beta_m=1, \beta_u=5, \beta_v = 10$. The test error on this data set is shown
in red (``Full data''). For the independent latent variables, we used
$\lambda^\mu\sim\mathcal{N}(0, 1)$ and $\nu^\mu=\pm 1 $ with equal
probability. For correlated latent variables, we set
$\nu^\mu=\mathrm{sign}(\lambda^\mu)$. We also show the test loss of the same
network evaluated on the following censored data sets:
\begin{itemize}
\item $\beta_m = 1, \beta_u = 0, \beta_v = 0$ (blue, ``mean only'')
\item $\beta_m = 1, \beta_u = 5, \beta_v = 0$ (green, ``mean + cov'')
\item A Gaussian equivalent model (orange), where we sample inputs for the
  spiked class from a Gaussian distribution with mean $\beta_m m$ and a
  covariance
  \begin{equation}
    \label{app:cov}
    \mathrm{cov}_{\mathbb{P}}(x, x) = \id + \beta_u u u^\top + \sqrt{\beta_u
      \beta_v (1-\gamma)^2}\;  \EE \lambda \nu \left( u v^\top + v u^\top \right)
  \end{equation}
\end{itemize}

\subsubsection{Teacher-student setup (figure \ref{fig:teacher-student})}

We trained the same two-layer network with the same hyperparameters as in
\cref{fig:figure1}. Defining the pre-activations $\lambda_i = \langle u_i, x
\rangle$ for the three spikes $u_i$, the tasks were:
\begin{description}
\item[A] inputs $x\sim \mathcal{N}(0, \id)$, teacher
  \begin{equation}
    y(x) = h_1(\lambda_1) + h_2(\lambda_2) +
    h_4(\lambda_3)
  \end{equation}
\item[B] For the simulation shown in A, We computed the absolute normalised
  overlaps $|\langle w_k, u_i \rangle| / \| w_k \|$ of all the first-layer
  weights $w_k$ with a given spike, and took the average over the five highest
  values.
\item[C] Same as in B, except that we sampled inputs from a normal distribution
  with zero mean and covariance $\id + \gamma (u v^\top + v u^\top)$,
  which introduces correlations between the pre-activations $\lambda_2$ and
  $\lambda_3$ of the teacher model.
\item[D] Same as in B, except that we trained on isotropic Gaussian inputs
  $\mathcal{N}(0, \id)$ and instead added a co-linear cross-term between spikes
  $u$ and $v$, akin to the cross-term that is found in our expansion of the
  mixed cumulant task,
  \begin{equation}
    y(x) = h_1(\lambda_1) + h_1(\lambda_2) h_1(\lambda_3) + h_2(\lambda_2) +
    h_4(\lambda_3)
  \end{equation}
\end{description}

